\documentclass[journal]{IEEEtran}

\usepackage[compress]{cite}
\usepackage{amsthm, amsmath, stackrel}
\usepackage{amssymb}
\usepackage{graphicx}
\usepackage{mathrsfs}
\usepackage{mathtools}
\usepackage{savesym}
\usepackage{cuted}
\usepackage{nccmath}
\usepackage{amsfonts}
\usepackage{epstopdf}
\usepackage{lipsum}
\usepackage{float}
\usepackage{stfloats}
\usepackage{url}
\usepackage{subcaption}
\usepackage{algorithm}
\usepackage{algpseudocode}
\usepackage{algcompatible}
\usepackage[super]{nth}
\usepackage[english]{babel}
\usepackage{multirow}
\usepackage{makecell}
\usepackage{soul}
\usepackage[dvipsnames]{xcolor}

\providecommand{\tabularnewline}{\\}

\providecommand{\U}[1]{\protect\rule{.1in}{.1in}}

\newtheorem{lemma}{Lemma}
\newtheorem{proposition}{Proposition}

\makeatother

\begin{document}

\title{Budgeted Online Selection of Candidate IoT Clients to Participate in Federated Learning}

\author{Ihab Mohammed,~\IEEEmembership{Member,~IEEE}, Shadha Tabatabai,~\IEEEmembership{Graduate Student Member,~IEEE},\\ Ala Al-Fuqaha,~\IEEEmembership{Senior Member,~IEEE}, Faissal El Bouanani,~\IEEEmembership{Senior Member,~IEEE}, Junaid Qadir,~\IEEEmembership{Senior Member,~IEEE}, Basheer~Qolomany,~\IEEEmembership{Member,~IEEE}, and
Mohsen~Guizani,~\IEEEmembership{Fellow,~IEEE}

\thanks{I. Mohammed is with the School of Computer
Sciences, Western Illinois University, Macomb, IL 61455 USA (e-mails: i-mohammed@wiu.edu).}

\thanks{S. Tabatabai is with the Department of Computer
Science, Western Michigan University and Computer Science Department, Al-Nahrain University (e-mails: shadhamuhinoo.tabatabai@wmich.edu).}

\thanks{A. Al-Fuqaha is with the Information and Computing Technology Division, Hamad Bin Khalifa University and Computer Science Department, Western Michigan University (e-mail:ala@ieee.org).}

\thanks{
F. El Bouanani is with Mohammed V University, Rabat, Morocco (e-mail: f.elbouanani@um5s.net.ma).}

\thanks{J. Qadir is with Information Technology University, Lahore, Pakistan  (e-mail: 
junaid.qadir@itu.edu.pk).}

\thanks{B. Qolomany is with  Department of Cyber Systems, College of Business \& Technology, University of Nebraska at Kearney, Kearney, NE 68849, USA (e-mail: qolomanyb@unk.edu).} 

\thanks{M. Guizani is with the Computer Science and Engineering Department, Qatar University, Doha, Qatar (e-mail: 
mguizani@ieee.org).}

}
\maketitle

\begin{abstract}

Machine Learning (ML), and Deep Learning (DL) in particular, play a vital role in providing smart services to the industry. These techniques however suffer from privacy and security concerns since data is collected from clients and then stored and processed at a central location. Federated Learning (FL), an architecture in which model parameters are exchanged instead of client data, has been proposed as a solution to these concerns. Nevertheless, FL trains a global model by communicating with clients over communication rounds, which introduces more traffic on the network and increases the convergence time to the target accuracy. In this work, we solve the problem of optimizing accuracy in stateful FL with a budgeted number of candidate clients by selecting the best candidate clients in terms of test accuracy to participate in the training process.
Next, we propose an online stateful FL heuristic to find the best candidate clients. Additionally, we propose an IoT client alarm application that utilizes the proposed heuristic in training a stateful FL global model based on IoT device type classification to alert clients about unauthorized IoT devices in their environment. To test the efficiency of the proposed online heuristic, we conduct several experiments using a real dataset and compare the results
against state-of-the-art algorithms. Our results indicate that the proposed heuristic outperforms the online random algorithm with up to 27\% gain in accuracy. Additionally, the performance of the proposed online heuristic is comparable to the performance of the best offline algorithm.

\end{abstract}

\IEEEpeerreviewmaketitle

\section{Introduction}
\label{introduction_sec}

The fourth industrial revolution (Industry 4.0) promises the provisioning of smart services that enhance the manufacturing process by utilizing emerging technologies such as Internet of Things (IoT) and Artificial Intelligence (AI) \cite{industrialAI_19} \cite{massiveIoT_20}. In particular, most of the recent advances in Industry 4.0 and AI are driven by Machine Learning (ML), a branch of AI, and more specifically by Deep Learning (DL) \cite{industrialAI_19} \cite{overview_19} \cite{towards_19}.

The ML and DL techniques however require a large amount of data for the training of their models. In particular, serious privacy and security concerns crop up when data is collected and processed from scattered organizations and users \cite{concept_application_19} \cite{vision_hype_19}. For instance, the prediction of patient mortality using Electronic Health Record (EHR) data dispersed over many hospitals is a complex undertaking due to the various privacy, security, regulatory, and operational issues \cite{FADL_19}. Additionally, the communication of potentially large amounts of data from the clients to a central server is costly and can choke the networks when limited bandwidth is available \cite{demo_19}. Such bottlenecks can be observed in Vehicular Edge Computing (VEC) where vehicles have to send their data such as images to roadside servers to build models, which results in the networks being greatly burdened \cite{veh_edge_comp_20}.

To address the issues of security, privacy, and excessive communication cost, the technique of Federated Learning (FL) \cite{google_17}, a distributed ML approach that runs on a server and multiple clients, was proposed. The server and the clients use the same model architecture. The server initiates the global model (i.e., the server model) and executes the following steps over several communication rounds \cite{demo_19} \cite{google_17}:

\begin{itemize}
\item The server sends the global model's parameters to some (or if possible all) clients;
\item Every participating client uses the received global parameters to train the local model using the local dataset;
\item Every participating client sends the local model parameters to the server;
\item The server aggregates the local parameters received from the clients to update the global model;
\item Eventually, the accuracy of the global model converges to some threshold. 
\end{itemize}

In FL, the server has no access to the client's local dataset since only the local model parameters are shared with the server. Consequently, privacy and security are preserved and communication cost is reduced. However, FL suffers from the following two problems \cite{challenges_19}:

\begin{itemize}
\item Convergence may take a long time, which increases the communication cost.
\item Clients have different computation, storage, and communication resources and different dataset sizes, which makes the task of selecting clients a challenge.
\end{itemize}

FL can be \textit{stateful} or \textit{stateless}. In stateful FL, a candidate client can participate in each of the communication and computation rounds used in training the global model and thus the state is preserved between rounds. Nevertheless, in stateless FL, a candidate client will likely participate in one communication and computation round to train the global model, which means in each round, new fresh candidate clients are utilized \cite{Advances_federated_2019}.

In this paper, we propose a stateful FL model with a budgeted number of candidate clients to overcome communication and computation constraints. In other words, from a total of $N$ candidate clients, we select the best $R<N$ candidate clients to participate in training the global model. Now, some candidate clients become available while others become offline or out of communication range over time. Also, we assume that not all candidate clients are available at the same time. Meaning that the problem of selecting $R$ candidate clients is an online problem. As a result, the selection of candidate clients is a challenge. In offline problems, information about all candidate clients are well known in advance rendering the problem of selecting the best $R$ candidate clients trivial. However, in online problems, once a candidate client becomes available then an irrevocable decision must be made on the selection of this candidate client without any prior knowledge about incoming candidate clients. Consequently, we propose a budgeted online selection algorithm that selects the best $R$ candidate clients based on their evaluated test accuracy. The proposed algorithm is inspired by the solution of the secretary problem.

The proposed algorithm can be used in different applications, particularly for online applications with intermittently available mobile clients. Once a client is available, a decision must be made on whether to utilize the client or not since the client may become unreachable like out of communication range or offline \cite{iot_Shadha_2020}. However, once the client is selected, the client will be utilized. Furthermore, decisions cannot be revoked in online applications but might be regretted. 

Detection and identification of unauthorized IoT devices are very important especially with the increase in the number of attacks on IoT devices \cite{forged_19}. Therefore, we propose a clients' alarm application that alerts clients about unauthorized IoT devices in their environment. Each client uses a local machine (i.e. server) to monitor the traffic generated by IoT devices in the environment and extract features based on IoT device behavior. Extracted features are used to identify the IoT device type by training an ML model on those features. This is known as \textit{IoT device type classification}. However, clients can not identify unknown IoT devices in their environment depending only on the local dataset. Therefore, clients subscribe to the alarm service provided by the server on the cloud that utilizes the proposed algorithm. Clients share their model's parameters with the server to train a global model capable of identifying unauthorized IoT devices.

The salient \textit{contributions of this paper} are:
\begin{itemize}
\item We propose a model for optimizing accuracy in stateful federated learning by selecting the best candidate clients based on test accuracy. We formulate the problem of maximizing the probability of selecting the best $R$ candidate clients based on test accuracy from $N$ total candidate clients as a secretary problem and analytically analyze the performance and provide proofs. 

\item We propose an online heuristic solution for optimal budgeted client selection based on test accuracy inspired by the secretary problem that works in stateful FL settings. To the best of our knowledge, this is the first work that utilizes online resources selection in federated learning.

\item We propose a client alarm application for identifying unauthorized IoT devices using the proposed algorithm and IoT device type classification. We conduct many experiments to evaluate the performance of the proposed heuristic against other state-of-the-art algorithms. Results show an improvement of up to 27\% in accuracy compared with the online random algorithm and an accuracy gain of approximately 10\%  compared with the offline best algorithm.
\end{itemize}

The organization of the remainder of the paper is as follows. A background regarding FL is discussed in section \ref{background}. Related literature is reviewed in Section \ref{related_work}. 
Section \ref{heuristic_solution} provides a heuristic solution. Section \ref{performance_analysis} provides performance proofs including analysis for the worst-case scenario. Experimental results are provided in Section \ref{experimental_results}, where we discuss the application, the dataset (and its preprocessing phases), and the conducted experiments. A discussion of the results and the salient lessons learned are provided in Section \ref{result_discussion}. Finally, the paper is concluded in Section \ref{conclusion_future_work} by summarizing this work and identifying future research directions.

\section{Background}
\label{background}

In this section, we introduce the properties and challenges of FL. Then, we review studies related to the online selection of resources using the optimal stopping theory and the secretary problem in particular.

\subsection{Federated Learning}

To understand the concepts of FL systems, Li \textit{et al.} \cite{vision_hype_19} provide a comprehensive study of FL systems. They categorize FL systems based on six features including machine learning model, communication architecture, data partition, privacy mechanism, motivation of federation, and scale of federation. Additionally, the authors present a summary of a comparison that includes 42 studies based on the six proposed features. In \cite{concept_application_19} and \cite{vision_hype_19}, the authors categorize FL based on data distribution as:

\begin{itemize}
\item \textit{Horizontal Federated Learning}: datasets of clients share the same feature space but with a small intersection in regards to the sample space.
\item \textit{Vertical Federated Learning}: datasets of clients share the same sample space but with a small intersection in regards to the feature space.
\item \textit{Hybrid Federated Learning} (\textit{Federated Transfer Learning}): datasets of clients have a small intersection in regards to both the feature and sample spaces.
\end{itemize}

We would like to emphasize that a Hybrid Federated Learning approach is used in this work. The main challenges in implementing FL, as described in \cite{challenges_19} and \cite{edge_comp_19}, are:

\begin{itemize}
\item \textit{Communication cost}: there could be many clients (millions) and the system may execute many rounds before converges to the required level of accuracy, which imposes an overload on the network.
\item \textit{Clients heterogeneity}: the system is heterogeneous and has clients with varying computation, storage, and communication capabilities. Also, the client datasets may differ in features and samples (i.e., the datasets may have statistical heterogeneity).
\item \textit{Privacy and security}: FL already protects clients' data by only sharing models' parameters. However, sensitive information may be revealed.
\end{itemize}

Researchers \cite{edge_comp_19} and \cite{towards_at_scale_19} have highlighted the importance of client selection for enhancing the performance of FL systems since it contributes to both communication cost and resource allocation.

Existing research on enhancing performance in FL follow one of the following approaches:

\begin{itemize}
\item \textit{Algorithm Optimization}: optimize the FL algorithm and perform more computation on clients to reduce the convergence time by reducing the number of rounds on the expense of more computation \cite{adaptive_19, comm_efficient_19, multi_objective_19, momentum_20, vertical_19, two_stream_18, client_edge_19, offloading_19}.

\item \textit{Selective Updates}: select only important updates from the clients or select the best clients in regards to the clients' resources and data size \cite{performance_19, client_sel_19, mitigation_19, crowd_machine_19, res_alloc_19, hybrid_fl_19}.

\item \textit{Model Compression}: reduce the amount of data exchanged between clients and the server \cite{comm_efficient_19, compression_19, strategies_18, expanding_18}.
\end{itemize}

\subsection{Secretary Problem}

The secretary problem, which is also known as the marriage problem, dowry problem, beauty contest problem, or Googol is a class of the optimal stopping decision problems. The secretary problem was first introduced by Martin Gardner back in 1960 \cite{who_solved_1989}. The classical secretary problem focuses on the selection of a secretary from a pool of candidates adhering to the following rules  \cite{who_solved_1989} \cite{best_worst_2019}:
\begin{itemize}
    \item The number $N$ of candidates is known,
    \item Only one candidate is to be chosen,
    \item Candidates are interviewed sequentially in random order,
    \item Each candidate must be accepted or rejected before interviewing the next one (with no provision for recalling rejected candidates later),
    \item Candidates are ranked from best to worst and the decision of accepting or rejecting a candidate depends on the relative ranks of candidates interviewed so far,
    \item The problem focuses on maximizing the probability of selecting the best candidate.
\end{itemize}

The solution of the secretary problem is for some integer $1 \leq \alpha < N$, reject the first $\alpha$ candidates then select the first candidate with rank better than of those observed candidates. The goal is to find the optimal $\alpha$ that maximizes the probability of selecting the best candidate. Actually, it has been proven that the optimal value for $\alpha$ is 0.367879 with optimal probability of $\frac{1}{e}$ \cite{who_solved_1989}. In other words, the probability of finding the best candidate is 37\% when rejecting the first 37\% of candidates and selecting the first candidate with ranking better than those observed ones.

The authors in \cite{secretary_review_1983} reviewed the extensions and generalizations of the secretary problem. They indicate that some researchers focus on the secretary problem when the number of candidates is unknown. Other researchers assume that candidates' ranks follow a specific distribution such as Poisson. Additionally, they show that some studies focus on selecting $R$ candidate instead of one.

In this paper, we are interested in studies of the secretary problem where $R$ top candidates are selected. In \cite{recog_1966}, the authors provide many variations of the secretary problem studied under different assumptions and one of these cases is for selecting $R$ candidates with one of the candidates as the best candidate. Kleinberg proposed an algorithm to maximize the sum of ranks of the $R$ selected candidates \cite{mit_auctions_2005}. The algorithm has two stages. In the first stage, the classical secretary algorithm is recursively applied on roughly the first half of candidates to select $l=R/2$ best candidates. In the second stage, the rank of the $l$th selected candidate in the first stage is used as a threshold for selecting $R/2$ candidates from the second half of candidates. The author states that the algorithm has a competitive ratio of $1-O(\sqrt{q/R})$. In \cite{online_auctions_2008}, the authors propose an algorithm to maximize the sum of the $R$ selected candidate. The algorithm rejects the first $\lfloor n/e \rfloor$ candidates and records the $R$ highest rankings in set $S$. Next, when a candidate with a rank higher than the minimum rank in $S$ is encountered, the candidate is selected and the minimum rank in $S$ is removed. This is repeated until either $S$ is empty or all candidates are reviewed. The authors indicate that the algorithm has a competitive ratio no worse than $e$ for all values of $R$.

The work in this paper is inspired by the aforementioned studies. However, this work is different in that we find the optimal stopping position $\alpha$, which we call $\alpha^*$, to maximize the probability of selecting the $R$ top candidates. We reject the first $\alpha^*$ candidates and record the best rank. Then, we use the best rank as a threshold in selecting the top $R$ candidates.

\section{Related Work}
\label{related_work}

FL is a hot research area that has recently grabbed the attention of many researchers. In this section, we list different approaches for enhancing the performance and discuss studies in each approach.

\subsection{Algorithm Optimization}

Some researchers work on optimizing the algorithm used in FL to reduce the convergence time and thus reduce the generated traffic in the network. Replacing the minibatch Stochastic Gradient Descent (mb-SGD) optimization model with \textit{Adam} has been studied in \cite{comm_efficient_19}. The authors propose CE-FedAvg, an algorithm that uses Adam optimization and compresses models before uploading to the server. The authors claim that using Adam optimization along with model compression reduces the convergence time by reducing the number of rounds and the amount of data exchanged between clients and the server. Using a multi-objective evolutionary algorithm with neural networks in FL has been studied in \cite{multi_objective_19}. The authors use the Elitist Nondominated Sorting Genetic Algorithm (NSGA-II) to minimize the communication cost at the expense of higher computation cost. In \cite{momentum_20}, researchers propose Momentum Federated Learning (MFL), which uses Momentum Gradient Descent (MGD) in every step of local updates rather than the first-order gradient descent. Authors state that since MGD consider preceding iteration, it converges faster than the traditional FL system.

Other researchers proposed algorithms that utilize the computation power on clients' machines to speed up the convergence process. To reduce the number of rounds, Liu et al. \cite{vertical_19} propose to use Federated Stochastic Block Coordinate Descent (FedBCD) algorithm in vertical FL, which let clients do multiple local model updates before syncing with each other. Authors in \cite{two_stream_18} claim that using two models in every client instead of a single model can reduce the number of rounds. Besides training the global model received from the server, each client trains another local model and uses the Maximum Mean Discrepancy (MMD) between the output of the two models. Using agents on edge nodes between clients and the server are studied in \cite{client_edge_19} and \cite{offloading_19}. Multiple agents perform partial model aggregation before communicating with the server to reduce the communication cost between clients and the server.

Other researchers study the trade-off between the number of iterations performed by clients to minimize the loss function and the frequency of global aggregation done by the server. In \cite{adaptive_19}, the authors compute the convergence bound of the gradient-descent algorithm then designed an algorithm that finds the best frequency of global aggregation based on system dynamics, model characteristics, and data distribution to minimize the consumed computation and communications. Qolomany et al. \cite{qolomany_particle_2020} proposed a Particle Swarm Optimization (PSO)-based technique to optimize the hyperparameter settings for the local ML models in an FL environment. They evaluated and compared the proposed PSO-based parameter optimization approach with the grid search technique. They found that the number of client-server communication rounds to explore the landscape of configurations to find the near-optimal parameter settings is greatly decreased by two orders of magnitude using the PSO-based approach compared to the grid search method. To deal with heterogeneous data inherent in federated networks, Li et al. \cite{li_federated_2020} proposed a modified version of FedAvg; namely, FedProx, that allows for variable amounts of work to be performed locally across devices, and relies on a proximal term which helps to improve the stability of the method against heterogeneous data.

\subsection{Selective Updates}

In \cite{performance_19}, the authors formulated a client selection and resource allocation optimization problem for FL in wireless networks to minimize the value of the loss function. They first derived an equation to represent the expected convergence rate of the FL algorithm. Next, they simplified the optimization problem as a mixed-integer nonlinear programming problem. Then for a given uplink resource block allocation and client selection, they compute the optimal transmit power. Finally, they transform the problem into a bipartite matching problem and use the Hungarian algorithm to find the optimal client selection and resource block allocation. Nishio and Yonetani \cite{client_sel_19} propose a new FL protocol named FedCS to enhance the efficiency of FL. The basic idea of the proposed protocol is to select clients based on their computation/communication capabilities and their data size instead of picking clients randomly. To reduce the communication overload, authors in \cite{mitigation_19} propose an approach that identifies clients with irrelevant updates and prevent those clients from uploading their updates to the server. In \cite{crowd_machine_19} and \cite{res_alloc_19}, the authors proposed the selection of clients based on the consumed energy in model's transmission and training, clients' distance from the server, and channel availability using Deep Reinforcement Learning (DRL) approach. Yoshida \textit{et al.} \cite{hybrid_fl_19} propose a hybrid FL approach based on the assumption that some clients share and upload their data to the server to improve the accuracy and mitigate the degradation resulted from non-independent-and-identically-distributed (non-IID) data. However, uploading clients' data to the server violates the rules of FL.

\subsection{Model Compression}

Sattler \textit{et al.} \cite{compression_19} proposed a new compression framework named Sparse Ternary Compression (STC). Authors claim that their compression framework performs better than other proposed methods in the literature in bandwidth-constrained learning environment. In \cite{strategies_18}, the authors propose to use structured updates (low rank and random mask) and force models to use these structures and also sketched updates with lossy compression before sending models to the server. On the other hand, Caldas \textit{et al.} \cite{expanding_18} apply lossy compression on the model sent from the server to the clients.

The related research discussed above has a high computational cost. The clients' intensive computation and algorithm optimization approach requires intensive computation. In addition to the extra computation required by the compression approach, it is best applied to models with large parameter vector such as images or models with many hidden layers. The presented studies using the selective updates approach either are too difficult to train (especially when a large number of clients are used as in DRL) or have no analysis and/or proofs for convergence.
In contrast, our proposed algorithm, which uses a selective updates approach, does not require intensive computations or a large parameter vector, and we also provide analysis and proofs demonstrating convergence.

\begin{table}[!th]
\caption{Summary of mathematical notations.}
\label{table_notations}
\centering
\begin{tabular}{p{0.16\textwidth}|p{0.28\textwidth}}
\hline
\thead{\textbf{\textit{Notation}}} & \thead{\textbf{\textit{Definition}}} \tabularnewline
\hline 
$N$ & Total number of candidate clients arriving until time $T$. Each candidate client is identified by an index in the interval $1..N$ \tabularnewline
\hline
$K$ & Number of communication rounds \tabularnewline
\hline
$E$ & Number of epochs \tabularnewline
\hline
$R$ & Number of required best candidate clients \tabularnewline
\hline
$\left(\mathcal{C}_{\ell}\right)_{1\leq\ell\leq N}$ & Set of candidate clients \tabularnewline
\hline
$\alpha$ & An index in the interval $1 .. N$ \tabularnewline
\hline
$\alpha^*$ & The optimal value of $\alpha$ \tabularnewline
\hline
$\mathcal{C}_{M}$ & The best candidate client in $\left[1..\alpha\right]$,
i.e., $1\leq M\leq\alpha$ \tabularnewline
\hline
$\left(i_{m}\right)_{1\leq m\leq R}$ & The set of $R$ positions corresponding to the top $R$ best candidate clients in the interval $\left[\alpha+1..N\right]$, better than $\mathcal{C}_{M}$\textbf{,} such that $\alpha+1\leq i_{m}\leq N$ and
$\mathcal{C}_{i_{m}}$ is worst than $\mathcal{C}_{i_{m-1}}$ for all $2\leq m\leq R$ \tabularnewline
\hline
$\left\{ \mathcal{A}^{(i)}\right\} _{\alpha+1\leq i\leq N}$ &
The set of events where ``the $i$th candidate better than $\mathcal{C}_{M}$ is selected''  \tabularnewline
\hline
$\left\{ \mathcal{B}^{(m,i)}\right\} _{1\leq m;\alpha+1\leq i\leq N}$ & The set of events where ``the $i$th candidate is the $m$th best one'' \tabularnewline
\hline
$\mathcal{E}_{R}$ & The event ``The $R$ best candidates in $\left[\alpha+1..N\right]$
better than $\mathcal{C}_{M}$ are selected'' occurring with the probability
$\Pr\left(\mathcal{E}_{R}\right)$; Obviously, $\stackrel[\ell=1]{R}{\cap}$ $\mathcal{E}_{R}=\stackrel[i_{R}\leq i_{R-1}..\leq i_{1}]{}{\cup}\stackrel[\ell=1]{R}{\cap}\left\{ \mathcal{A}^{\left(i_{\ell}\right)}\cap\mathcal{B}_{R}^{\left(\ell,i_{\ell}\right)}\right\} $
with $i_{1}$ and $i_{R}$ correspond to the best and the worst combinations, among $R$ selected ones, respectively \tabularnewline
\hline
$r_1$ & Minimum number of best candidate clients.
\tabularnewline
\hline
$r_2$ & Maximum number of best candidate clients.
\tabularnewline
\hline
$\mathcal{P}^{(r_{1},r_{2})}$ & Probability to select the \textbf{$R$} best candidates $\left(i_{m}\right)_{1\leq m\leq R}$ where $r_{1}\leq R\leq r_{2}$ and $r_{1},r_{2}<N.\mathcal{P}^{(r_{1},r_{2})}=\sum_{R=r_{1}}^{r_{2}}\Pr\left(\mathcal{E}_{R}\right)$.  \tabularnewline
\hline 
$P_i$ & Probability of selecting candidate $i$
\tabularnewline
\hline
\end{tabular}
\end{table}

\section{Proposed Client Selection Solution}
\label{heuristic_solution}

\subsection{System Model}
\label{system_model}

We assume $N$ candidate clients and one server. Also, we assume a budget of $R$ candidate clients. The nature of the proposed model is online since some clients become available while others become unreachable or offline over time. Consequently, the server must make an irrevocable decision to accept (i.e. select) or reject a candidate client once a candidate client becomes available. The server runs the proposed heuristic (explained in the next section), which initializes the global model, selects the $R$ best candidate clients based on their test accuracy, and then train the global model using the selected candidate clients in $K$ communication rounds. Each selected candidate client trains the local model in $E$ epochs using the local dataset but with the global model parameters. Moreover, we assume the datasets of candidate clients are different in size. Therefore, we use the terms \textit{fat clients} and \textit{thin clients} to point to candidate clients with different sizes of datasets. We note that in some literature, the terminology of elephants (instead of fat) and mice (instead of thin) is used instead \cite{mice_2001}. For the convenience, we have listed the main mathematical notations used in this paper in Table \ref{table_notations}.

\subsection{Problem Formulation}

The problem we tackle in this paper is to select the best set of candidate clients that provide higher test accuracy when training the global model using their local dataset. This problem is similar to the famous \textit{secretary problem}, which aims to maximize the probability of selecting the maximum element from a randomly ordered sequence \cite{secretary_08}. The secretary problem is formulated as a linear programming problem as follows \cite{secretary_lp_2014}:

\begin{align*}
&\max \frac{1}{N}\cdot \sum_{1=1}^{N} iP_i\\
&s.t. \quad \forall 1 \leq i \leq N \quad i \cdot P_i \leq 1 - \sum_{j=1}^{i-1} P_j \\
&\qquad \forall 1 \leq i \leq N \quad P_i \geq 0
\end{align*}

In the traditional secretary problem, the objective function aims to maximize the probability of selecting the best candidate. However, instead of selecting one element, in this problem, $R$ elements must be selected. The secretary problem is one scenario of the \textit{optimal stopping theory}. In the secretary problem, an employer wants to hire a secretary and there are $N$ candidates. The employer cannot assess the quality of a candidate until after the end of the interview and have to make an irrevocable hiring decision. Thus, the employer may end up hiring a candidate before interviewing the rest of the candidates and the hiring of the best candidate is not guaranteed.

Our solution is inspired by the \textit{secretary problem}. The quality of a client is determined by its test accuracy. We evaluate the test accuracy (i.e., quality) of the first $\alpha^*$ (see section \ref{performance_analysis}) clients and reject them all. Then, select the next $R$ clients with test accuracy better than the best test accuracy of the first $\alpha^*$ clients, and if none is found then select the last clients.

\subsection{Proposed Algorithm}
\label{proposed_algorithm}

The proposed heuristic identifies the best test accuracy among the first few available candidate clients then use this test accuracy as a threshold for accepting or rejecting candidate clients available later. The heuristic accepts the parameters $N$, $R$, $r_1$, $r_2$, $K$, $E$, and $\delta$ as explained in Algorithm \ref{alg_proposed} and consists of \textit{three stages} that run every $\delta$ time units to update the global model.

In the \textit{first stage} (Algorithm \ref{alg_proposed}, lines 2 through 11), the value of $\alpha^*$ (discussed in Section \ref{performance_analysis}) is computed based on the value of $r_1$ and $r_2$ using equation (\ref{alphastar}). The first $\alpha^*$ candidate clients that are available are then tested to determine the best test accuracy. However, none of those candidate clients are accepted. Whenever a candidate client becomes available, the server initializes the global model and communicates with the candidate client to evaluate its test accuracy. Testing is performed by sending the initialized global model's parameters from the server to the candidate client for one communication round so that the candidate client trains the local model with these parameters using the local dataset. Then, the candidate client sends back the updated parameters to the server. The server evaluates the received parameters (i.e., no averaging is applied since only one candidate client is involved) using the test dataset to determine the test accuracy of the candidate client. After testing $\alpha^*$ candidate clients, the server selects the best test accuracy to be used as a threshold in the second section.

\begin{algorithm}[!t]
\footnotesize
\caption{Proposed heuristic}
\label{alg_proposed}
\begin{algorithmic}[1]
\STATEx \textbf{Input}: $N$ (expected number of clients), $R$ (number of selected candidate clients), $r_1$, $r_2$ (to compute $\alpha^*$), $K$ (number of communication rounds), $E$ (number of epochs per client), and $\delta$
\STATEx \textbf{Output}: Trained global model

\FOR {every $\delta$ time units}
\STATEx
\begin{center}
\textbf{\textit{// Find best test accuracy for first $m$ candidate clients, $m:1..\alpha^*$}}
\end{center}
\STATE Initialize the global model
\STATE Compute $\alpha^*$ based on $r_1$, $r_2$, and $N$ using equation (\ref{alphastar})
\STATE Set $A_b$, best test accuracy = 0
\FOR {$m$ = 1 to $\alpha^*$}
\STATE Test client $\mathcal{C_{M}}$ and record $A_m$, its test accuracy
\IF {$A_m > A_b$}
\STATE Set $A_b = A_m$
\ENDIF
\STATE Reject candidate client $C_m$
\ENDFOR

\STATEx
\begin{center}
\textbf{\textit{// Find $R$ best candidate clients}}
\end{center}
\STATE Set $S_b$, set of best candidate clients = []
\STATE Set $N_b$, number of best candidate clients found = 0
\FOR {$m=\alpha^*+1$ to $N$}
\IF {$N_b$ = $R$}
\STATE Reject candidate client $C_m$
\ELSIF {$(N-m) \leq (R-N_b)$}
\STATE Accept candidate client $C_m$ and add it to $S_b$
\STATE Increment $N_b$ by 1
\ELSE
\STATE Test client $C_m$ and record $A_m$, its test accuracy
\IF {$A_m > A_b$}
\STATE Accept candidate client $C_m$ and add it to $S_b$
\STATE Increment $N_b$ by 1
\ENDIF
\ENDIF
\ENDFOR

\STATEx
\begin{center}
\textbf{\textit{// Start training}}
\end{center}
\FOR {$k=1$ to $K$}
\STATE Send global model to all candidate clients in $S_b$
\STATE Candidate clients train global model on local dataset for $E$ epochs
\STATE Server average aggregated model parameters from candidate
\STATEx \quad \quad clients in $S_b$
\ENDFOR
\ENDFOR
\end{algorithmic}
\end{algorithm}

In the \textit{second stage} (Algorithm \ref{alg_proposed}, lines 12 through 27), whenever a candidate client becomes available, it gets tested in the same way explained in the first section. Next, the server accepts (i.e., selects) the candidate client only if its test accuracy is greater than the best test accuracy found in the first section. Nonetheless, if the number of available candidate clients is less than the number of required candidate clients (i.e., $R$) then the server has no choice but to select those remaining candidate clients. In the worst-case scenario, the candidate client with the best test accuracy is met during the first section. Consequently, all candidate clients met early in the second stage are rejected for having a test accuracy less than that of the best test accuracy found in the first section. As a result, the server is forced to accept all candidate clients that are met at the end of the second section. In fact, in the worst-case scenario, the proposed heuristic behaves similarly to the random algorithm explained in Section \ref{result_discussion}. 

In the \textit{third stage} (Algorithm \ref{alg_proposed}, lines 28 through 32), once the best candidate clients are identified, the global model is trained using the selected candidate clients for $K$ communication rounds as described in Section \ref{introduction_sec}.

\subsection{Illustrative Example}

To understand the proposed algorithm in more depth, we present an example where we observe one run cycle (when $\delta$ is 1) of the proposed algorithm overtime (see Figure \ref{fig:HeuristicExample}). Assume a total of 10 candidate clients becoming available over time during the observed period (i.e. $N=10$). We refer to candidate $i$ as $\mathcal{C}_i$. Additionally, we set the budget to 2 candidate clients (i.e. $R$ = 2), which means that we want to select the best 2 candidate clients for training the global model. Moreover, we set $E$, the number of epochs, to 3 and set $K$, the number of communication rounds between the server, and selected candidate clients for training the global model to 20. Also, the value of $\alpha^*$ is computed based on equation (\ref{alphastar}) in section \ref{optimal_alpha} (assuming $r_1$ is 1 and $r_2$ is 2) and its value is 2.

\begin{figure}[!h]
\centering
\includegraphics[width=0.45\textwidth,trim={-3 0 0 -3},clip]{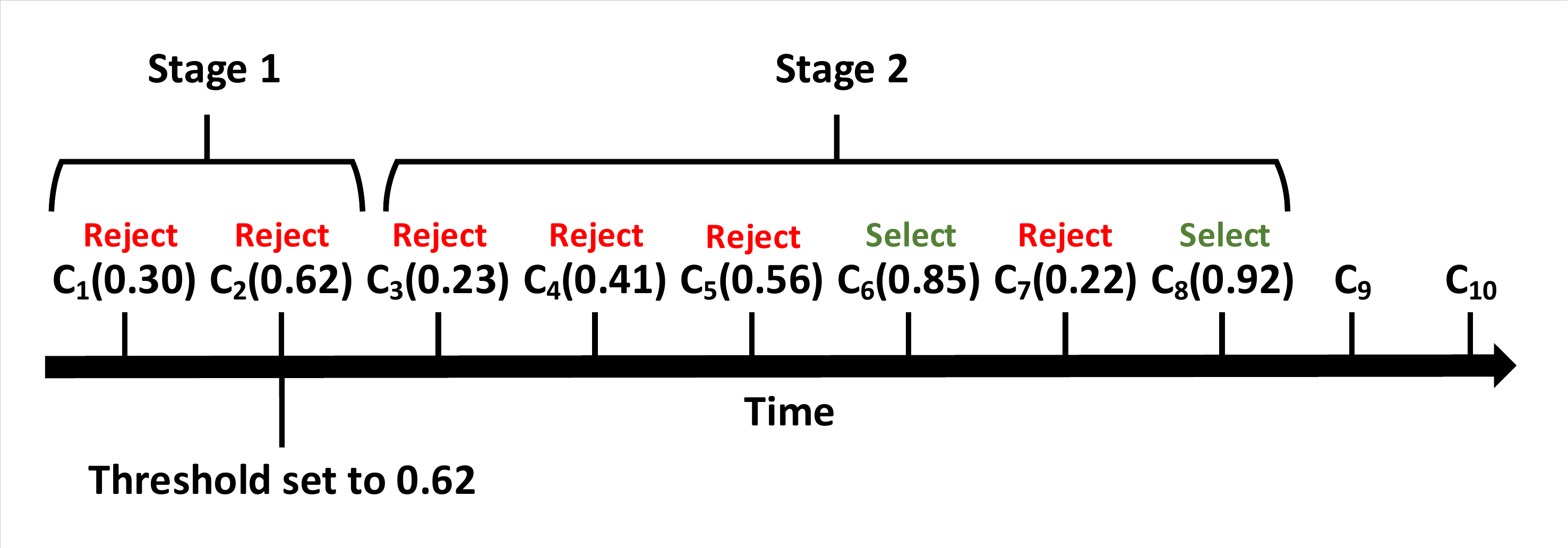}
\caption{An illustrative example of the proposed algorithm.}
\label{fig:HeuristicExample}
\end{figure}

When a candidate client becomes available then (1) the proposed algorithm initiates the global model's parameters then sends them to the candidate client, (2) The candidate client trains the local model for $E$ epochs using the received parameters from the server on the local dataset, (3) the candidate client sends the updated parameters to the server, (4) the server evaluate the accuracy of the candidate client by testing the global model using the updated parameters on the test dataset. Then, the proposed algorithm must make an irrevocable decision on whether to use this client or not based on its evaluated test accuracy.

The proposed algorithm runs in three stages. In the first stage, the proposed algorithm communicates with the first $\alpha^*$ (i.e. 2) candidate clients and evaluate their test accuracy to determine the best test accuracy, which is used as a selection threshold with the rest of candidate clients that become available later. Thus, when $\mathcal{C}_1$ becomes available, the proposed algorithm communicates with $\mathcal{C}_1$ then evaluates its test accuracy and finds it 0.30. The proposed algorithm sets its selection threshold to 0.30 and rejects $\mathcal{C}_1$. Next, $\mathcal{C}_2$ becomes available and the proposed algorithm communicates with $\mathcal{C}_2$ then evaluates its test accuracy and finds it 0.62. The proposed algorithm updates its selection threshold to 0.62 as illustrated in Fig. \ref{fig:HeuristicExample} where $\mathcal{C}_i(x)$ represents candidate client $i$ with evaluated test accuracy $x$ (test accuracy is a number between 0 and 1, where 0 means the trained model fails to identify all test samples while 1 means the trained model identifies all test samples successfully).

In the second stage, the proposed algorithm will continue to communicate with any candidate client that becomes available and evaluate its test accuracy to decide on the selection of this candidate client. This process continues as shown in Fig. \ref{fig:HeuristicExample} until the proposed algorithm selects 2 candidate clients and as follows:

\begin{itemize}
    \item $\mathcal{C}_3$ becomes available and its test accuracy is 0.23 and thus gets rejected.
    \item $\mathcal{C}_4$ becomes available and its test accuracy is 0.41 and thus gets rejected.
    \item $\mathcal{C}_5$ becomes available and its test accuracy is 0.56 and thus gets rejected.
    \item $\mathcal{C}_6$ becomes available and its test accuracy is 0.85 and thus gets selected.
    \item $\mathcal{C}_7$ becomes available and its test accuracy is 0.2 and thus gets rejected.
    \item $\mathcal{C}_8$ becomes available and its test accuracy is 0.92 and thus gets selected.
    \item The server is not going to communicate with $\mathcal{C}_9$ and $\mathcal{C}_{10}$ once they are available since the proposed algorithm has already selected two candidate clients.
\end{itemize}

In the third stage, the proposed algorithm trains the global model using $\mathcal{C}_6$ and $\mathcal{C}_8$ with $K$ communication rounds but without initiating the global model in every round. A best-case scenario is presented in this example, but a worst-case scenario can occur if the test accuracy of $\mathcal{C}_2$ is evaluated and found as 0.93. In this case, the proposed algorithm rejects both $\mathcal{C}_6$ and $\mathcal{C}_8$. Eventually, the proposed algorithm will have to communicate with the last two clients ($\mathcal{C}_{7}$ and $\mathcal{C}_{10}$) and selects both.

\section{Performance Analysis}
\label{performance_analysis}

The performance of the proposed algorithm explained in the previous section depends vitally on the optimal value of $\alpha$, which is $\alpha^*$. In this section, we derive an equation for computing the value of $\alpha^*$ and prove its validity. This equation is plugged in the first stage of the proposed algorithm as mentioned in section \ref{proposed_algorithm}. Finally, we analytically analyze the performance of the proposed algorithm in worst-case scenario.

\subsection{Optimal Value for $\alpha$}
\label{optimal_alpha}

By assuming (i) $M$ and $i_{m}$ positions are not known in advance, (ii) the candidates can arrive in any order, and (iii) $N,\alpha>>R$, we aim to find the optimum value $\alpha^{*}$, depending on both, allowing to maximize $\mathcal{P}^{(r_{1},r_{2})}$.

\begin{lemma} The following summation 
\begin{equation}
\mathcal{K}(R,\alpha)=\sum_{i_{R}=\alpha+1}^{N-R+1}\frac{1}{i_{R}-1}\sum_{i_{R-1}=i_{R}+1}^{N-R+2}\frac{1}{i_{R-1}-1}..\sum_{i_{1}=i_{2}+1}^{N}\frac{1}{i_{1}-1},
\end{equation}
can be tightly approximated by 
\begin{equation}
\mathcal{K}(R,\alpha)\approx\frac{\left(\log\frac{N}{\alpha}\right)^{R}}{R!}.\label{eq:ind}
\end{equation}

\end{lemma}

\begin{proof} Let us proceed by induction. one can ascertain that for $R=1$, the summation $\sum_{i_{1}=\alpha+1}^{N}\frac{1}{i_{1}-1}$
can be approximated by the $\int_{\alpha}^{N}\frac{dt}{t}=\log\frac{N}{\alpha}$,
confirming \eqref{eq:ind}.

Let us assume that \eqref{eq:ind} holds for $R-1$. One obtains 

\begin{align}
\mathcal{K}(R,\alpha)= & \sum_{i_{R}=\alpha+1}^{N-R+1}\frac{1}{i_{R}-1}\mathcal{K}(R-1,i_{R})\nonumber \\
\approx & \frac{1}{\left(R-1\right)!}\sum_{i_{R}=\alpha+1}^{N-R+1}\frac{\left(\log\frac{N}{i_{R}}\right)^{R-1}}{i_{R}-1}\nonumber \\
\approx & \frac{1}{\left(R-1\right)!}\int_{\alpha}^{N-R+1}\frac{\left(\log\frac{N}{t}\right)^{R-1}}{t}dt.
\end{align}

Finally, taking into account that $R<<N$ (i.e., $N-R+1\approx N$), it follows that

\begin{align}
\mathcal{K}(R,\alpha) & \approx\frac{1}{R!}\left[-\left(\log\frac{N}{t}\right)^{R}\right]_{\alpha}^{N}\\
 & \approx\frac{1}{R!}\left(\log\frac{N}{\alpha}\right)^{R},
\end{align}
which concludes the proof.
\end{proof}

\begin{proposition} For all positive numbers $r_{1},r_{2}<<\alpha,N$, the approximation

\begin{align}
\mathcal{P}^{(r_{1},r_{2})} & \approx\frac{\alpha}{N}\sum_{R=r_{1}}^{r_{2}}\frac{1}{R!}\left(\log\frac{N}{\alpha}\right)^{R},\label{Prob-1}
\end{align}
holds, and the optimum value maximizing such probability is 

\begin{equation}
\alpha^{*}=N\exp\left(-\left(\frac{r_{2}!}{\left(r_{1}-1\right)!}\right)^{\frac{1}{r_{2}-r_{1}+1}}\right).
\label{alphastar}
\end{equation}

\end{proposition}

\begin{proof}

Given that the indices of the selected candidates are sorted in increasing order of candidates' accuracies, $\mathcal{E}_{R}$ can be broken into $R$ exclusives events as follows

\begin{itemize}
\item Candidate client $M$ is the best one in $[1,i_{R}-1]$ \textbf{and} 

\item Candidate clients $i_{m}$ are the best ones in $[1,i_{m-1}-1]$,\textbf{
$2\leq m\leq R$} \textbf{and} 

\item Candidate client $i_{1}$ is best one in $[1,N]$. 
\end{itemize}

Consequently,

\begin{multline}
\Pr\left(\mathcal{E}_{R}\right)=\sum_{i_{R}=\alpha+1}^{N-R+1}\sum_{i_{R-1}=i_{R}+1}^{N-R+2}.. \\
\sum_{i_{1}=i_{2}+1}^{N}\Pr\left(\underset{\mathcal{D}_{R}}{\underbrace{\stackrel[\ell=1]{R}{\cap}\left\{ \mathcal{A}^{\left(i_{\ell}\right)}\cap\mathcal{B}_{R}^{\left(\ell,i_{\ell}\right)}\right\} }}\right).
\label{eq0-1-1}
\end{multline}

With the aid of the Bayes's rule, $\mathcal{D}_{R}$ can be rewritten as

\begin{multline}
\Pr\left(\mathcal{D}_{R}\right) =\\
\Pr\left(\left.\mathcal{A}^{\left(i_{R}\right)}\right|\mathcal{B}_{R}^{\left(R,i_{R}\right)}\cap\mathcal{D}_{R-1}\right)\Pr\left(\left.\mathcal{B}_{R}^{\left(R,i_{R}\right)}\right|\mathcal{D}_{R-1}\right).
\label{entireterm}
\end{multline}

The probability to select the $R$th best one among $\left[1..N\right]\setminus\{i_{1},i_{2},..,i_{R-1}\}$ is 
\begin{align}
\Pr\left(\left.\mathcal{B}_{R}^{\left(R,i_{R}\right)}\right|\mathcal{D}_{R-1}\right) & =\frac{1}{N-R+1},\label{secondterm-1}
\end{align}

\begin{flushleft}
with the conditional probability in \eqref{entireterm} can be evaluated as 
\end{flushleft}

\begin{equation}
\Pr\left(\left.\mathcal{A}^{\left(i_{R}\right)}\right|\mathcal{B}_{R}^{\left(R,i_{R}\right)}\cap\mathcal{D}_{R-1}\right)=\frac{\alpha}{i_{R}-1}\stackrel[\ell=1]{R-1}{\prod}\frac{1}{i_{\ell}-1}.
\label{eq1-1-1}
\end{equation}

Substituting \eqref{eq1-1-1}, \eqref{secondterm-1}, and \eqref{entireterm} into \eqref{eq0-1-1}, one obtains

\begin{align}
\Pr\left(\mathcal{E}_{R}\right) & =\frac{\alpha}{N-R+1}\mathcal{K}(R,\alpha).
\label{eq:conditional generalized}
\end{align}

Leveraging \textbf{Lemma 1} and noting that $N-R+1\approx N$, \eqref{Prob-1} is obtained. Now, defining $x=\alpha/N$ (i.e., $0\leq x\leq1$),
the two first derivatives of $\mathcal{P}^{(r_{1},r_{2})}$ with respect
to $x$ can be expressed as

\begin{equation}
\frac{\partial\mathcal{P}^{(r_{1},r_{2})}}{\partial x} =\frac{\left(-\log x\right)^{r_{2}}}{r_{2}!}-\frac{\left(-\log x\right)^{r_{1}-1}}{\left(r_{1}-1\right)!},
\end{equation}

\begin{equation}
\frac{\partial^{2}P^{(r_{1},r_{2})}}{\partial x^{2}} =-\frac{1}{x}\left[\frac{\left(-\log x\right)^{r_{2}-1}}{\left(r_{2}-1\right)!}-\frac{\left(-\log x\right)^{r_{1}-2}}{\left(r_{1}-2\right)!}\right].
\end{equation}

Thus, by solving $\frac{\partial\mathcal{P}^{(r_{1},r_{2})}}{\partial x}=0$ and setting $\alpha^{*}=Nx^{*}$, we get \eqref{alphastar}. Moreover, it can be easily checked that the second derivative evaluated at $x^{*}$

\begin{align}
&\frac{\partial^{2}P^{(r_{1},r_{2})}}{\partial x^{2}} = \nonumber \\
&-\frac{\left(-\log x\right)^{r_{1}-2}}{x^{*}\left(r_{1}-2\right)!} \left[\frac{\left(r_{1}-2\right)!}{\left(r_{2}-1\right)!}\underset{=\frac{r_{2}!}{\left(r_{1}-1\right)!}}{\underbrace{\left(-\log x^{*}\right)^{r_{2}-r_{1}+1}}}-1\right] \nonumber \\ 
& =-\left(r_{2}-r_{1}+1\right)\frac{\left(-\log x\right)^{r_{1}-1}}{x^{*}\left(r_{1}-1\right)!},
\end{align}
is negative as $r_{2}>r_{1}$ and $x^{*}\leq1$, which completes the proof.
\end{proof}

Table \ref{table_choosing_alpha} summarizes some values of the optimal number $\alpha^{*}$
along with the aforementioned maximum probability for various values of $r_{1}$ and $r_{2},$ when $N=1000$. Note that the probability \eqref{Prob-1} is an increasing function on $r_{2}$, while its maximum value is not monotone as it depends also on $r_{1}$ as summarized in Table I. It can be seen also that:
\begin{itemize}
 \item The smaller $r_{1}$ is, the greater the optimal value ($\alpha^{*}=Nx^{*}$).
 \item For a fixed $r_{1}$, the larger $r_{2}$ is, the smaller $\alpha^{*}$.
\end{itemize}

Fig. \ref{fig:alpha_prob} shows that the probability of selecting the best $R$ clients is higher when the value of $\alpha$ is small.

\begin{figure}[htbp]
\centering
\includegraphics[width=0.45\textwidth]{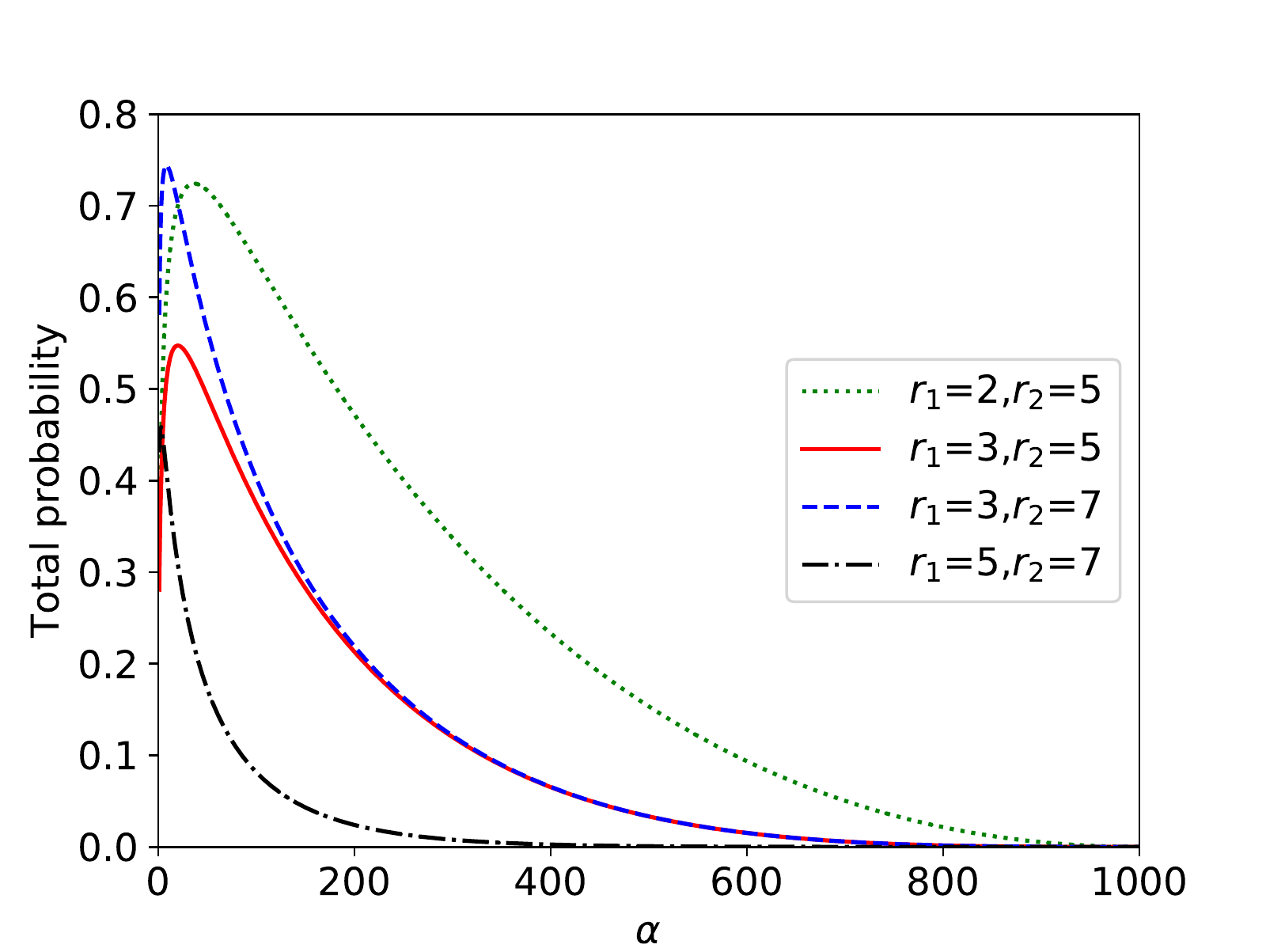}
\caption{Effects of $\alpha$ on the probability of selecting the best clients.}
\label{fig:alpha_prob}
\end{figure}

\begin{table}[th]
\caption{Choosing $\alpha^{*}$ that maximizes the probability to select \textbf{$R$}
best candidates such that $r_{1}\protect\leq R\protect\leq r_{2}$ and $N=1000$.}
\label{table_choosing_alpha}
\centering %
\begin{tabular}{c|c|c|c|c}
\hline 
\thead{$r_{1}$} & \thead{$r_{2}$} & \thead{$\alpha^{*}$} & \thead{Percentage (\%) $x^{*}=\frac{\alpha^{*}}{N}$} & \thead{$\mathcal{P}_{\max}^{(r_{1},r_{2})}$}
\tabularnewline
\hline 
$2$ & $2$ & $135.3353$ & $13.53$ & $0.2707$\tabularnewline
\hline 
$2$ & $3$ & $49.7871$ & $4.97$ & $0.4481$\tabularnewline
\hline 
$2$ & $4$ & $0.3355$ & $0.03$ & $0.0966$\tabularnewline
\hline 
$3$ & $3$ & $49.7871$ & $4.97$ & $0.2240$\tabularnewline
\hline 
$3$ & $4$ & $2.4788$ & $0.24$ & $0.2231$\tabularnewline
\hline 
\end{tabular}
\end{table}

\subsection{Worst-Case Analysis (Competitive Ratio Analysis)}

The worst-case scenario is encountered when the proposed heuristic does not find candidates that exceed $\mathcal{C_M}$ from index $\alpha^{*}$ until $N$. The competitive ratio in the worst-case scenario is computed over all possible input sequences as the maximum ratio of the gain of the online algorithm and the optimal offline algorithm \cite{competitive_ratio_2013}.

\begin{proposition}
The heuristic's worst-case performance has a competitive ratio of $O(1)$ when $R$ is proportional to $N$.
\end{proposition}

\begin{proof}
Let ALG be the proposed heuristic and OPT be the optimal algorithm. 

The worst-case happens when the highest element appears before index $\alpha^{*}$. In that case, the proposed algorithm randomly selects candidate clients from index $\alpha^{*}+1$ until $N$. A candidate client within this range of indices is selected with probability$\frac{R}{N-\alpha^{*}}$. Consequently, the following proof is concluded as follows:
\\
\newline
\begin{align*}
Com_r = \frac{ALG}{OPT}
= \frac{R}{N-\alpha^{*}}\\
\end{align*}
Thus, $Com_r$, the competitive ratio, becomes $\mathcal{O}(1)$ when $R$ is proportional to $N$.
\end{proof}

\section{Experimental Settings}
\label{experimental_results}

In this section, we describe the application proposed in this paper in detail first. Next, we describe the dataset used in the simulation and describe the dataset preparation phases used to transform the raw dataset into $N$ candidate clients' datasets. Finally, we discuss conduced experiments.

\begin{figure}[!h]
\centering
\includegraphics[width=0.44\textwidth]{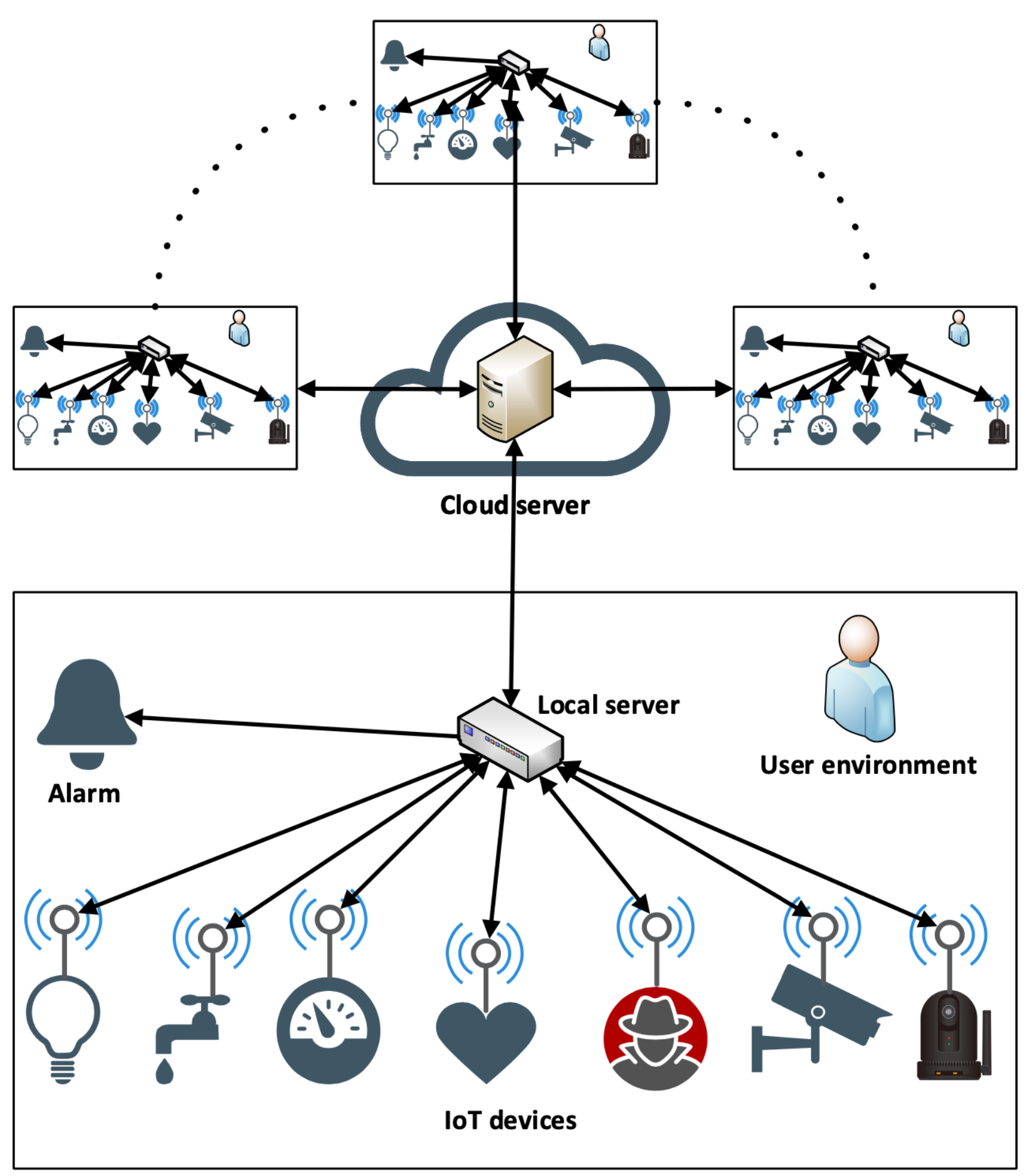}
\caption{An illustration of the clients alarm application. The cloud server running the proposed algorithm communicates with the local servers of the best subscribed clients to train the global model.}
\label{fig:application}
\end{figure}

\subsection{Use Case: IoT Device type Classification}
\label{problem}

\begin{figure*}[!h]
\centering
\includegraphics[width=1\textwidth]{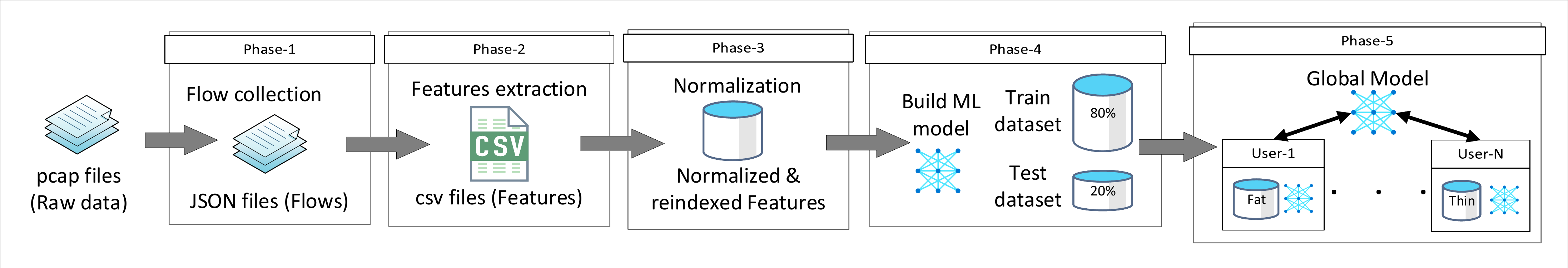}
\caption{Dataset preprocessing phases (through which the raw dataset is transformed to the $N$ candidate clients' datasets).}
\label{fig:SimulationPhases}
\end{figure*}

IoT devices perform specific tasks, which makes their network behavior predictable \cite{rec_analysis_18}. There are plenty of studies on IoT device type classification or fingerprinting in the literature \cite{rec_analysis_18, ecosystems_18, forged_19, home_iot_18, feature_ranking_19, telescopes_18, sentinel_17, smart_cities_17, profilIoT_17, tmc_19}. Those studies concentrate on identifying IoT devices type for different reasons including security, access control, provisioning, resource allocation, and management \cite{ecosystems_18}. Actually, most of those studies concentrate on security in response to recent incidents \cite{iot_optimalLoadDDOS_2020}, \cite{iot_intrusion_2020}. In one incident, thousands of IoT devices including surveillance cameras are used for Distributed Denial of Service (DDoS) attack \cite{forged_19}. Therefore, we propose a client alarm application based on IoT device type classification in FL settings to identify unauthorized IoT devices. The IoT device type classification is inspired by the work in \cite{tmc_19}. We aim to use the proposed application as a use-case to test the performance of the proposed heuristic.

The proposed application consists of $N$ candidate clients, a main server in the cloud, and an alarm mechanism. Each client's environment has several IoT devices, a local machine (i.e., the local server), and an alarm device as shown in Fig. \ref{fig:application}. The alarm can be a physical device or software that delivers email, text messages, or any other form of notification to the client. The local server monitors the traffic generated by IoT devices, extract features, and build a local dataset. Then, train the local model using the local dataset. However, training on local dataset is not sufficient to identify unknown IoT devices in the environment. As a result, the clients subscribe to the alarm service provided by the server through the use of FL. The server is responsible for running the proposed FL algorithm. Also, the server and clients cooperate to build a global model capable of classifying devices used by participating clients. In other words, clients can use the global model to identify unknown devices from the knowledge of other clients.

The proposed algorithm in the server trains a global model by sharing only the model's parameters with clients and thus preserving the security and privacy of clients. The process of training the global model is repeated every $\delta$ time units to make sure that the new clients, and clients with the new installed IoT devices, are considered and included.

Employing all clients in the training process produce high traffic, which overloads the network. Additionally, this might be infeasible since some clients are not available all the time. However, selecting clients with high accuracy contribution to the global model training enhance the classification accuracy, which is done by the proposed heuristic.

\subsection{Dataset Details and Preprocessing Phases}

To test the performance of the proposed heuristic, we use a real dataset collected by researchers from the University of New South Wales (UNSW), Sydney, Australia \cite{tmc_19}. The dataset is created using 28 IoT devices and also some non-IoT devices installed in a lab on the campus of the university. Trace data are captured over 6 months between October 1, 2016 and April 13, 2017. However, only 20 days of trace data are available for the public. Raw data consisting of packet headers and payload information are captured using the \verb|tcpdump| tool installed on the gateway. The dataset is available as a set of pcap (packet capture) files and also as a set of CSV (comma-separated values) files. The dataset consists of 20 pcap files, one file per day.

The raw dataset is processed in five phases (illustrated in Fig \ref{fig:SimulationPhases}) in order to create $N$ candidate client's datasets to simulate FL settings as described next.

In the \textit{\textbf{flows collection phase}} (\textbf{Phase-1}), we collect flows from raw data in pcap files using the \verb|joy| tool developed by Cisco Systems \cite{joy}. \verb|Joy| is a data collection tool that reads the data from raw traffic (or from pcap files) and produces a JavaScript Object Notation (JSON) file with a summary of the traffic data in the form of flows. We create a bash script that uses the \verb|joy| tool to process the pcap files and produce JSON files. Each JSON file contains flows related to a specific IoT device based on the MAC address listed in Table \ref{table_macs}, which includes names of devices and their MAC addresses as indicated in the dataset's website \cite{tmc_19}. To filter by MAC address, we use the Berkeley/BSD Packet Filter syntax supported by the \verb|joy| tool through the data feature options. Each flow in the resultant JSON file has a flow key that includes the source and destination addresses, and the source and destination port and protocol numbers. Each flow also contains number of bytes, number of packets, start time, and end time. Additionally, \verb|joy| can be configured to save more information per flow. Algorithm \ref{alg_flows_collection} describes the flow collection process. Also, the script is available on GitHub \cite{my_github}.

\begin{table}[htbp]
 \scriptsize
 \centering
 \caption{Names and MAC addresses of the used IoT devices}
 \label{table_macs}
 \begin{tabular}{c|c}
 \hline
 \thead{\textbf{\textit{IoT device name}}} & \thead{\textbf{\textit{MAC address}}}
 \\ \hline
 Amazon Echo & 44:65:0d:56:cc:d3
 \\ \hline
 August Doorbell Cam & e0:76:d0:3f:00:ae
 \\ \hline
 Awair air quality monitor & 70:88:6b:10:0f:c6
 \\ \hline
 Belkin Camera & b4:75:0e:ec:e5:a9
 \\ \hline
 Belkin Motion Sensor & ec:1a:59:83:28:11
 \\ \hline
 Belkin Switch & ec:1a:59:79:f4:89
 \\ \hline
 Blipcare BP Meter & 74:6a:89:00:2e:25
 \\ \hline
 Canary Camera & 7c:70:bc:5d:5e:dc
 \\ \hline
 Dropcam & 30:8c:fb:2f:e4:b2
 \\ \hline
 Google Chromecast & 6c:ad:f8:5e:e4:61
 \\ \hline
 Hello Barbie & 28:c2:dd:ff:a5:2d
 \\ \hline
 HP Printer & 70:5a:0f:e4:9b:c0
 \\ \hline
 iHome PowerPlug & 74:c6:3b:29:d7:1d
 \\ \hline
 LiFX Bulb & d0:73:d5:01:83:08
 \\ \hline
 NEST Smoke Sensor & 18:b4:30:25:be:e4
 \\ \hline
 Netatmo Camera & 70:ee:50:18:34:43
 \\ \hline
 Netatmo Weather station & 70:ee:50:03:b8:ac
 \\ \hline
 Phillip Hue Lightbulb & 00:17:88:2b:9a:25
 \\ \hline
 Pixstart photo frame & e0:76:d0:33:bb:85
 \\ \hline
 Ring Door Bell & 88:4a:ea:31:66:9d
 \\ \hline
 Samsung Smart Cam & 00:16:6c:ab:6b:88
 \\ \hline
 Smart Things & d0:52:a8:00:67:5e
 \\ \hline
 TP-Link Camera & f4:f2:6d:93:51:f1
 \\ \hline
 TP-Link Plug & 50:c7:bf:00:56:39
 \\ \hline
 Triby Speaker & 18:b7:9e:02:20:44
 \\ \hline
 Withings Baby Monitor & 00:24:e4:10:ee:4c
 \\ \hline
 Withings Scale & 00:24:e4:1b:6f:96
 \\ \hline
 Withings Sleep Sensor & 00:24:e4:20:28:c6
 \\
 \hline
 \end{tabular}
\end{table}

\begin{algorithm}[!t]
\footnotesize
\caption{Flows collection algorithm}
\label{alg_flows_collection}
\begin{algorithmic}[1]
\STATEx \textbf{Input}: dataset pcap files.
\STATEx \textbf{Output}: JSON files.

\FOR {each pcap file as $pFileName$}
\STATE Open $pFileName$ for reading
\STATE Set $deviceCo = 1$
\STATE Set $json = pFileName+deviceCo$
\FOR {each MAC address in Table \ref{table_macs} as $mac$}
\STATE Run \verb|joy| with $pFileName$ as input, $json$ as output, and
\STATEx \quad \quad $mac$ as the host MAC address
\STATE Set $deviceCo = deviceCo + 1$
\ENDFOR
\STATE Close $pFileName$
\ENDFOR
\end{algorithmic}
\end{algorithm}

\begin{table}[!htbp]
 \footnotesize
 \centering
 \caption{IoT device features.} 
 \label{table_features}
 \begin{tabular}{c|p{4.8cm}}
 \hline
 \thead{\textbf{\textit{Feature}}} & \thead{\textbf{\textit{Description}}}
 \\ \hline
 $totalSleepTime$ & Total time of no activity
 \\ \hline
 $totalActiveTime$ & Total time of activity
 \\ \hline
 $totalFlowVolume$ & Number of bytes (sent/received) by the IoT device
 \\ \hline
 $flowRate$ & Total flow volume divided by total active time
 \\ \hline
 $avgPacketSize$ & Number of bytes sent or received divided by no. of packets sent or received
 \\ \hline
 $numberOfServers$ & Number of servers \\ & Excluding DNS (53) and NTP (123)
 \\ \hline
 $numberOfProtocols$ & Number of protocols \\ & based on destination port number
 \\ \hline
 $numberOfUniqueDNS$ & Number of unique DNS requests
 \\ \hline
 $DNSinterval$ & Total time for using DNS
 \\ \hline
 $NTPinterval$ & Total time for using NTP
 \\
 \hline
 \end{tabular}
\end{table}

\begin{algorithm}[!h]
\footnotesize
\caption{Features extraction algorithm}
\label{alg_features_extraction}
\begin{algorithmic}[1]
\STATEx \textbf{Input}: JSON files.
\STATEx \textbf{Output}: features CSV file.

\STATE Open features file for writing
\STATE Set $maxPeriod$ = 10 minutes
\FOR {each JSON file}
\STATE Open JSON file for reading
\STATE Read $deviceID$ from JSON file
\STATE Set $totalSleepTime = 0$; $totalActiveTime = 0$
\STATE Set $totalFlowVolume = 0$; $totalPackets = 0$
\STATE Set $numberOfServers = 0$; $numberOfProtocols = 0$
\STATE Set $numberOfUniqueDNS = 0$; $DNSinterval = 0$
\STATE Set $NTPinterval = 0$; $lastFlowEndTime = 0$

\FOR {each flow}
\STATE Set \#flow = flow number
\STATE Set $flowTime = flowEndTime - flowStartTime$
\STATE Set $totalActiveTime$ = $totalActiveTime$ + $flowTime$
\STATE Set $totalFlowVolume$ = $totalFlowVolume$ + number of bytes in the flow
\STATE Set $totalPackets$ = $totalPackets$ + number of packets in the flow
\IF {port in flow is not recorded before}
\STATE Set $numberOfProtocols$ = $numberOfProtocols$ + 1
\STATE Record port
\ENDIF
\IF {port in flow = 53}
\STATE Set $DNSinterval$ = $DNSinterval$ + $flowTime$
\IF {DNS query in flow is not recorded before}
\STATE Set $numberOfUniqueDNS$ = $numberOfUniqueDNS$ + 1
\STATE Record DNS query
\ENDIF
\ELSIF {port in flow = 123}
\STATE Set $NTPinterval$ = $NTPinterval$ + $flowTime$
\ELSE
\IF {destination address in flow is not recorded before}
\STATE Set $numberOfServers$ = $numberOfServers$ + 1
\STATE Record destination address
\ENDIF
\ENDIF
\IF {\#flow = 1}
\STATE Set $startTime = flowStartTime$
\ELSE
\STATE Set $totalSleepTime$ = $totalSleepTime$ + 
\STATEx \qquad \qquad $(flowStartTime - lastFlowEndTime)$
\IF {$flowEndTime - startTime \geq maxPeriod$}
\STATE Set $flowRate = 0$
\IF {$totalActiveTime \geq 0$}
\STATE Set $flowRate = totalFlowVolume/totalActiveTime$
\ENDIF
\STATE Set $avgPacketSize = 0$
\IF {$totalPackets \geq 0$}
\STATE Set $avgPacketSize$=$totalFlowVolume/totalPackets$
\ENDIF
\STATE Add a record to features file with features and $deviceID$
\STATE Reinitialize all features variables
\ENDIF
\ENDIF
\STATE Set $lastFlowEndTime = flowEndTime$
\ENDFOR
\STATE Close JSON file
\ENDFOR
\STATE Close features file
\end{algorithmic}
\end{algorithm}

In the \textit{\textbf{features extraction phase}} (\textbf{Phase-2}), we extract features from the flows stored in JSON files. Inspired by a previous study \cite{tmc_19}, we analyze the flows and extract features as listed in Table \ref{table_features}. Features are saved in a CSV file with the first 10 columns for features and the last column for the labels, which are the IoT device IDs. Algorithm \ref{alg_features_extraction} shows the steps used in the extraction process. In addition, the Python code for extracting the features is made available on GitHub \cite{my_github}.

In the \textit{\textbf{normalization} phase} (\textbf{Phase-3}), we first normalize all features by transforming features' values to be between 0 and 1 using the \verb|MinMaxScaler| function from the \verb|scikit-learn| library \cite{scikit-learn}. Second, to ensure that samples are distributed randomly, we randomly re-index all normalized features in the dataset.

In the \textit{\textbf{ML model design phase}} (\textbf{Phase-4}), we split the dataset into two parts: the training dataset (80\% of the original) and the test dataset (20\% of the original). To ensure a fair comparison between the proposed algorithm and other algorithms, the test dataset is stored on the server. We then design a Deep Neural Network (DNN)-based ML model with three layers, each having 25 neurons. The first two layers use the ReLu activation function, while the last layer uses a softmax activation function. Adam optimizer is utilized for optimization.

Finally, in the \textit{\textbf{dataset splitting phase}} (\textbf{Phase-5}), we create $N$ datasets to represent local datasets for the $N$ candidate clients. Each of the $N$ datasets is created randomly from the training dataset. To reflect a real scenario, we ensure that those datasets do not have the same size. The majority of candidate clients possess a small amount of the dataset while the minority of candidate clients possess large portions of the dataset. Consequently, the fat clients constitute 20\% of candidate clients and each fat client has about 10\% of the training dataset selected randomly. On the other hand, the thin clients constitute 80\% of candidate clients and each thin client has about 1\% of the training dataset randomly selected. All these parameters along with other FL parameters are listed in Table \ref{table_FL_parameters}.

\begin{table}[!h]
 \footnotesize
 \centering
 \caption{FL Parameters.}
 \label{table_FL_parameters}
 \begin{tabular}{c|c}
 \hline
 \thead{\textbf{\textit{Parameter}}} & \thead{\textbf{\textit{Value(s)}}}
 \\ \hline
 Batch size & 3
 \\ \hline
 $E$ (Epochs) & 8
 \\ \hline
 $K$ (Communication rounds) & 20
 \\ \hline
 Test dataset & 20\% of the dataset
 \\ \hline
 Train dataset & 80\% of the dataset
 \\ \hline
 Number of fat clients & 20\% of $N$
 \\ \hline
 Number of thin clients & 80\% of $N$
 \\ \hline
 Fat client dataset & 10\% of the train dataset
 \\ \hline
 Thin client dataset & 1\% of train dataset
 \\ \hline
 \end{tabular}
\end{table}

\begin{table}[htbp]
 \footnotesize
 \centering
 \caption{Simulation Parameters.}
 \label{table_Sim_parameters}
 \begin{tabular}{c|c|c}
 \hline
 \thead{\textbf{\textit{Sym.}}} & \thead{\textbf{\textit{Parameter}}} & \thead{\textbf{\textit{Value(s)}}}
 \\ \hline
 $N$ & No. of candidate clients & 
 (100, 200, 400, 800, and 1600)
 \\ \hline
 $R$ & No. of best candidate clients &
 (10, 20, 30, 40, and 50)
 \\ \hline
 $r_1$ & Minimum no. of & 1 \\
 &  best candidate clients &
 \\ \hline
 $r_2$ & Maximum no. & (1, 2, 3, 4, and 5) \\
 & of best candidate clients &
 \\ \hline
 \end{tabular}
\end{table}

\subsection{Experiments}

After Phase-5, $N$ candidate clients' datasets are formed to simulate a Hybrid FL setting, which is utilized in the conducted experiments. To measure the performance of the proposed heuristic, we compare the results of two algorithms against the proposed heuristic. The two algorithms are the online random algorithm and the offline best algorithm. The online random algorithm selects and rejects candidate clients randomly. On the other hand, the offline best algorithm is an offline algorithm that can work with all candidate clients at the same time. In other words, the offline best algorithm does not have to wait for clients to be available over time and instead have the advantage of working with all candidate clients at the same time. The offline best algorithm creates a sorted list of all (i.e., $N$) candidate clients based on accuracy and selects the top $R$ candidate clients, which are mostly fat candidate clients.

We conduct 125 experiments to test the performance of the proposed heuristic. In all these experiments, we fixed the number of communication rounds to 20; the number of epochs per client to 8; and the batch size to 3 (as shown in Table \ref{table_FL_parameters}). We are not interested in optimizing the aforementioned parameters since the goal of this paper is not to achieve the highest accuracy possible, but to investigate the ability of the proposed heuristic compared against the state-of-the-art algorithms. Thus, we vary the number of clients $N$, number of selected candidate clients $R$, and $r_2$ (used to compute the value of $\alpha^*$) each with five different values as indicated in Table \ref{table_Sim_parameters}. Experimenting with different values of $N$, $R$, and $r_2$ is vital to truly test the abilities of the proposed heuristic.

The values in Table \ref{table_Sim_parameters} are not selected arbitrarily. We test with values of $N$ that vary from hundreds to thousands by doubling the numbers to see how this increase affects the performance. As for $R$, we test with different values in tens and noticed that raising $R$ more is not interested since the performance of all algorithms converges as explained later. Setting $r_1$ to one is a must since we need to select the first best candidate client. Additionally, we noticed that raising the value of $r_2$ to more than 5 will results in a very low $\alpha^*$ especially when $N$ is 100. In other words, setting $r_2$ to higher numbers will reduce the search size to zero candidate clients since $\alpha^*$ will be close to zero.

\begin{figure*}[htbp]
\centering
\begin{subfigure}[b]{0.32\textwidth}
\includegraphics[width=\textwidth]{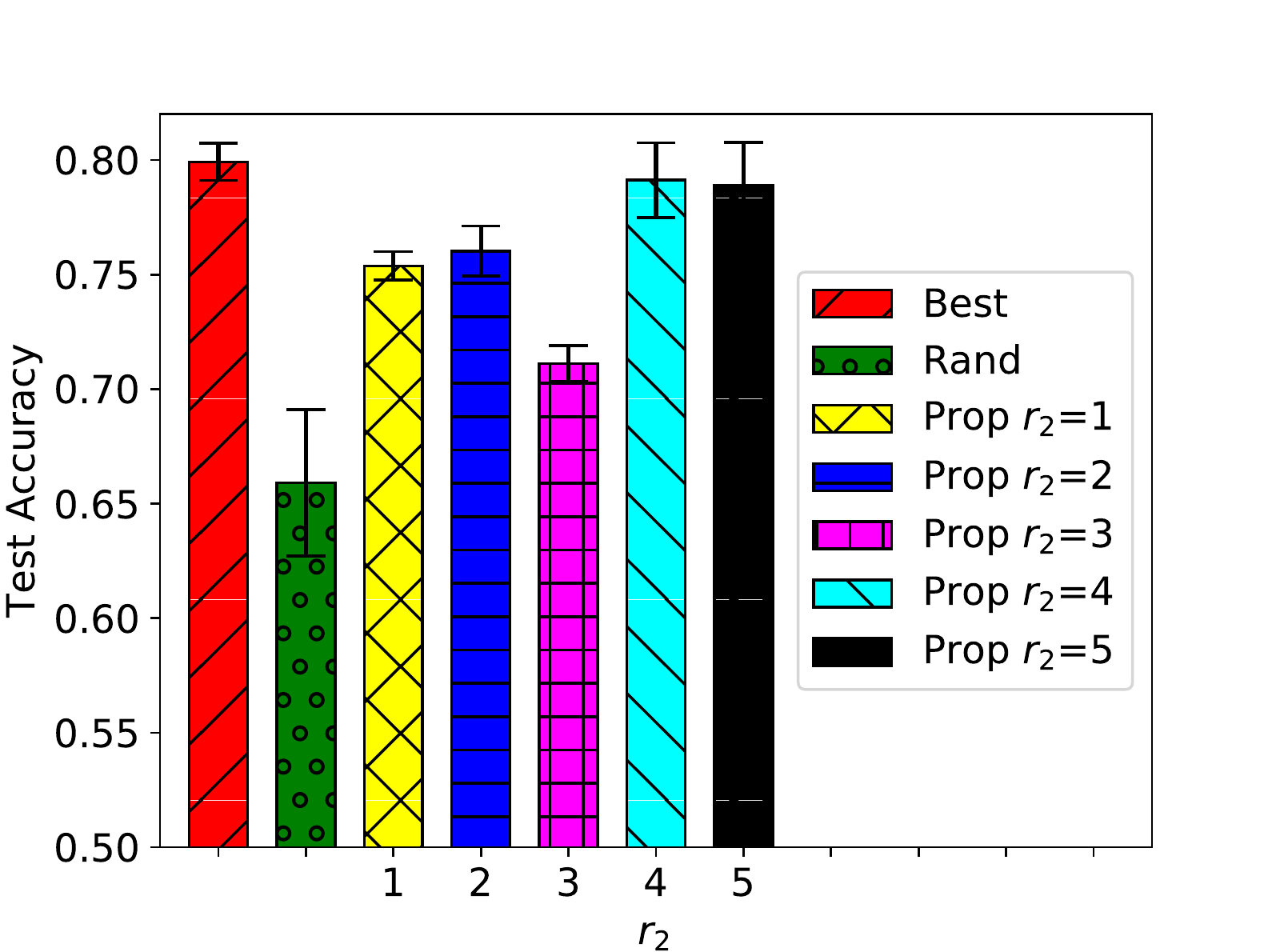}
\caption{Test accuracy.}
\label{fig:A_r2}
\end{subfigure}
\begin{subfigure}[b]{0.32\textwidth}
\includegraphics[width=\textwidth]{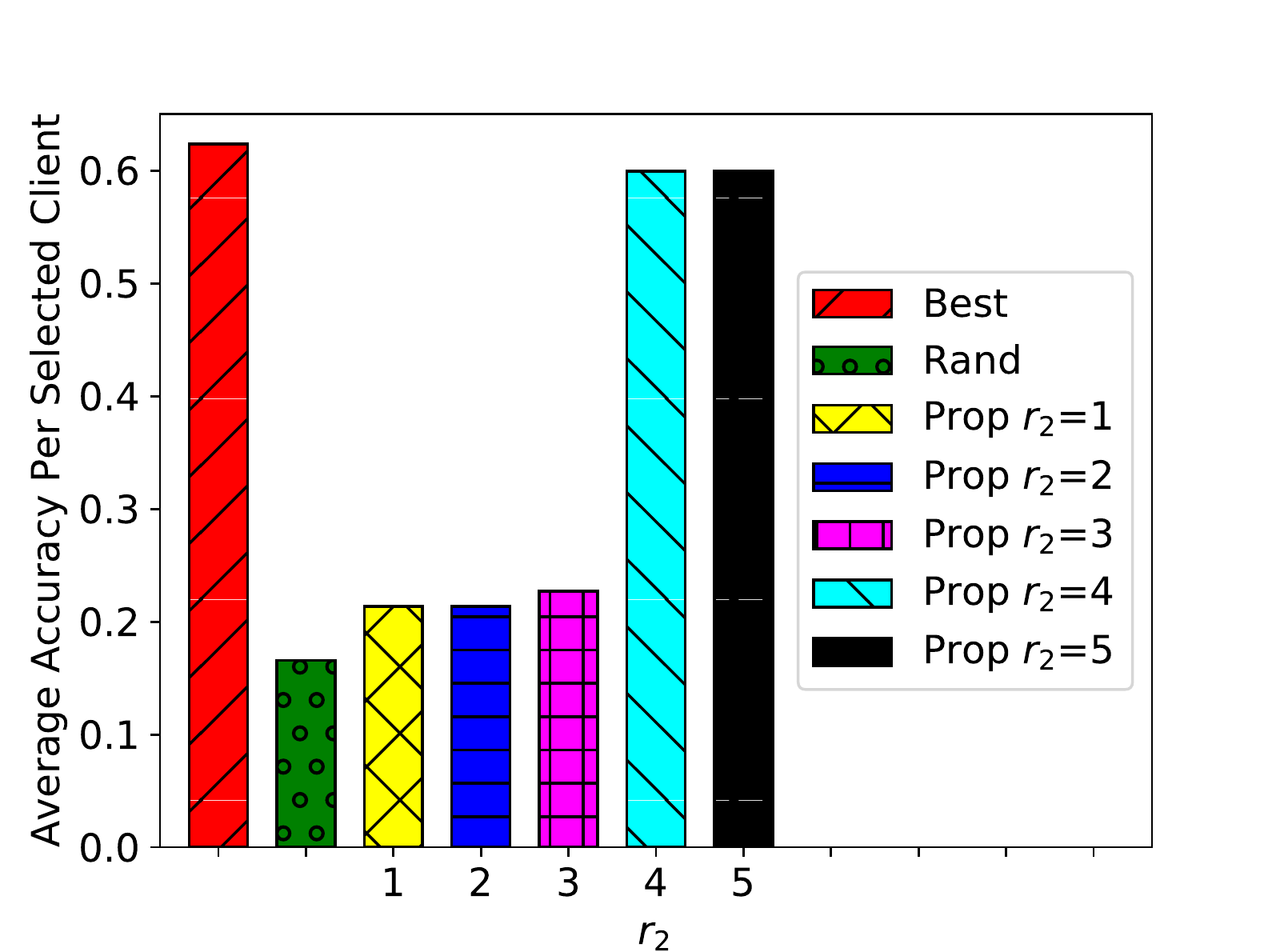}
\caption{Average accuracy per selected client.}
\label{fig:Ac_r2}
\end{subfigure}
\begin{subfigure}[b]{0.32\textwidth}
\includegraphics[width=\textwidth]{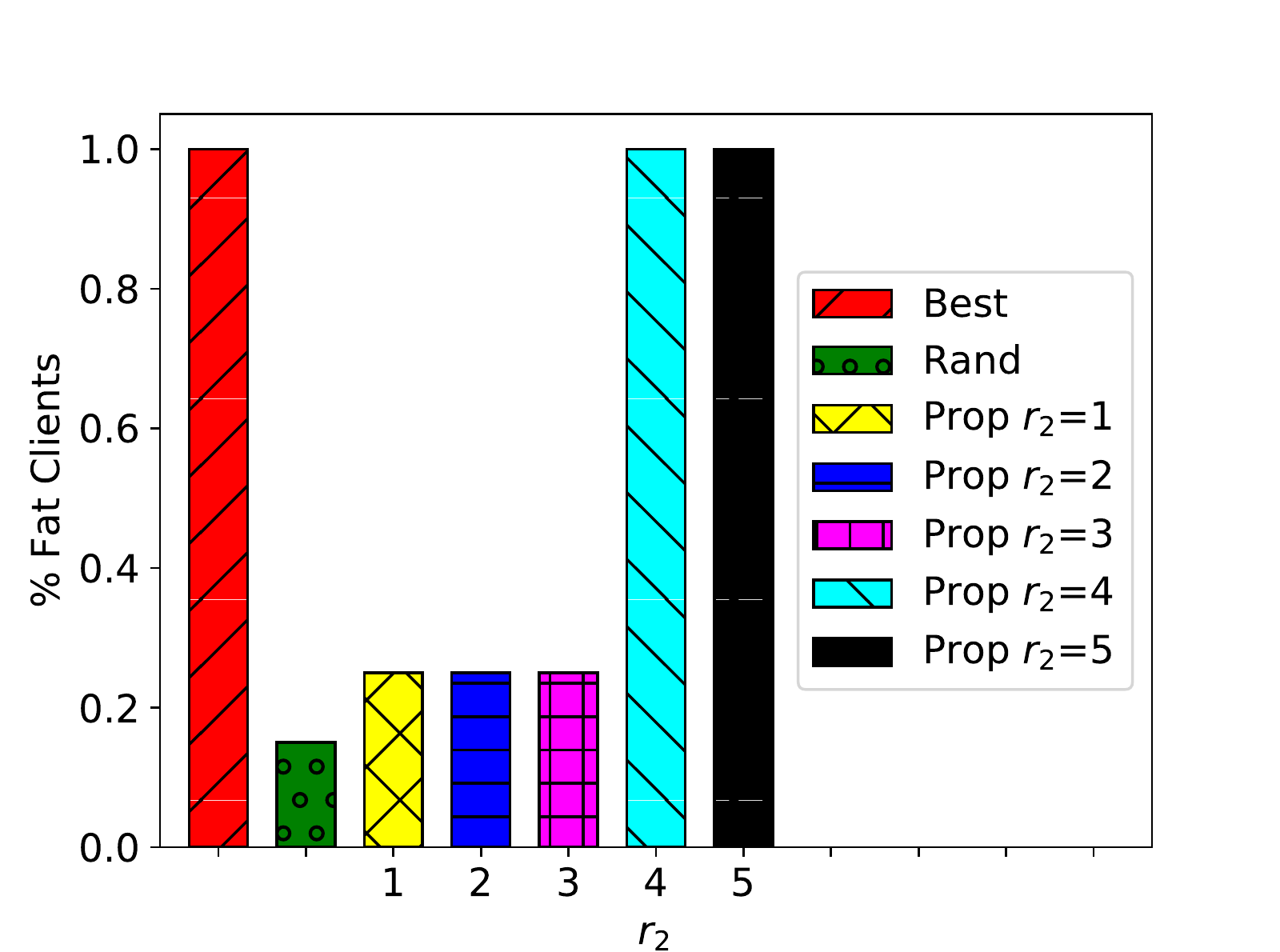}
\caption{\% Fat clients.}
\label{fig:fat_r2}
\end{subfigure}
\caption{Performance of algorithms for different $r_2$ values (1, 2, 3, 4, 5) while fixing $N$, number of clients, to 400 and $R$, number of selected clients, to 20. \textit{Our proposed algorithm performs better than the random algorithm approaching the performance of the best algorithm as $r_2$ is increased.}}
\label{fig:r2}
\end{figure*}

\begin{figure*}[htbp]
\centering
\begin{subfigure}[b]{0.32\textwidth}
\includegraphics[width=\textwidth]{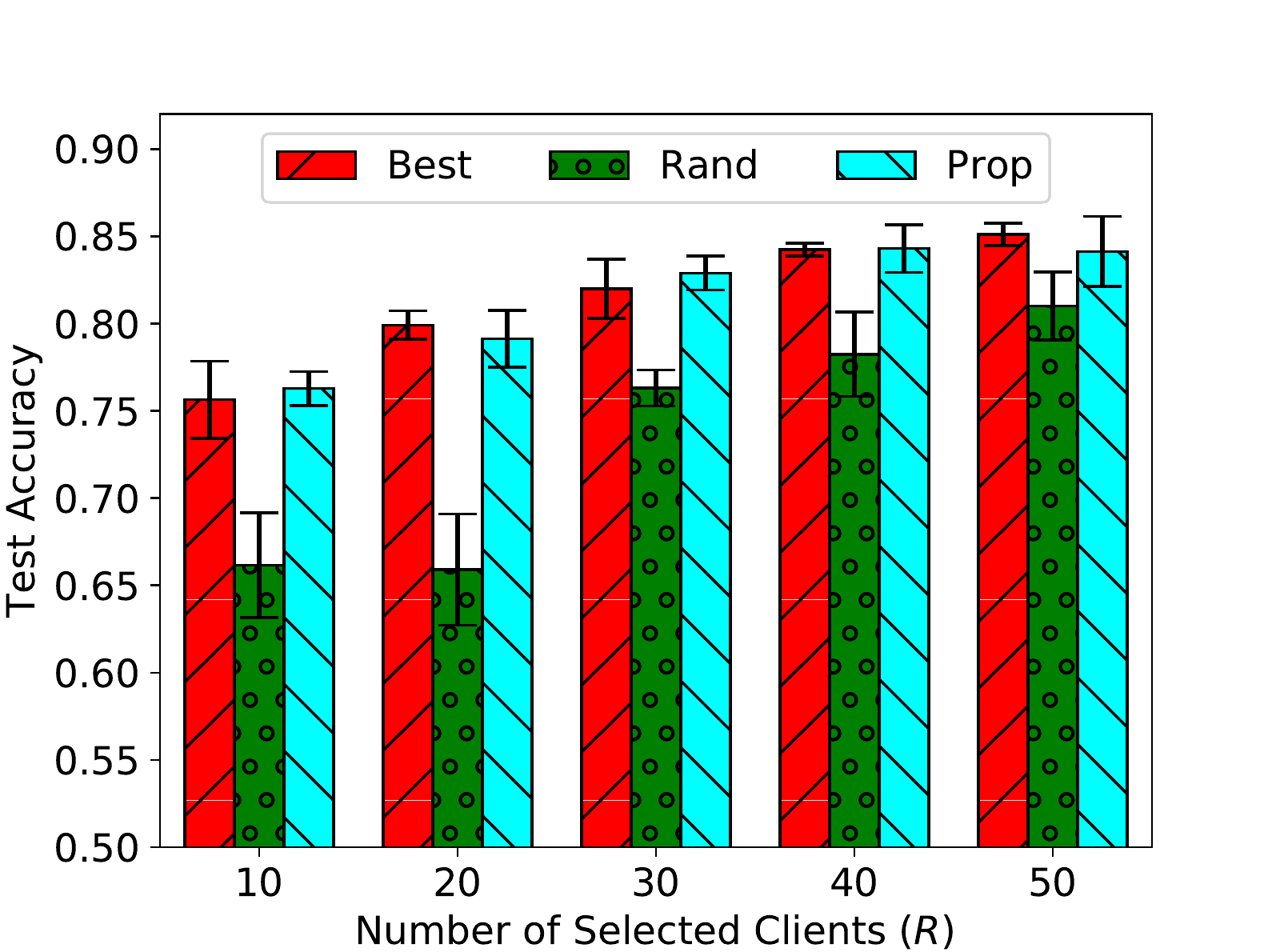}
\caption{Test accuracy.}
\label{fig:A_R}
\end{subfigure}
\begin{subfigure}[b]{0.32\textwidth}
\includegraphics[width=\textwidth]{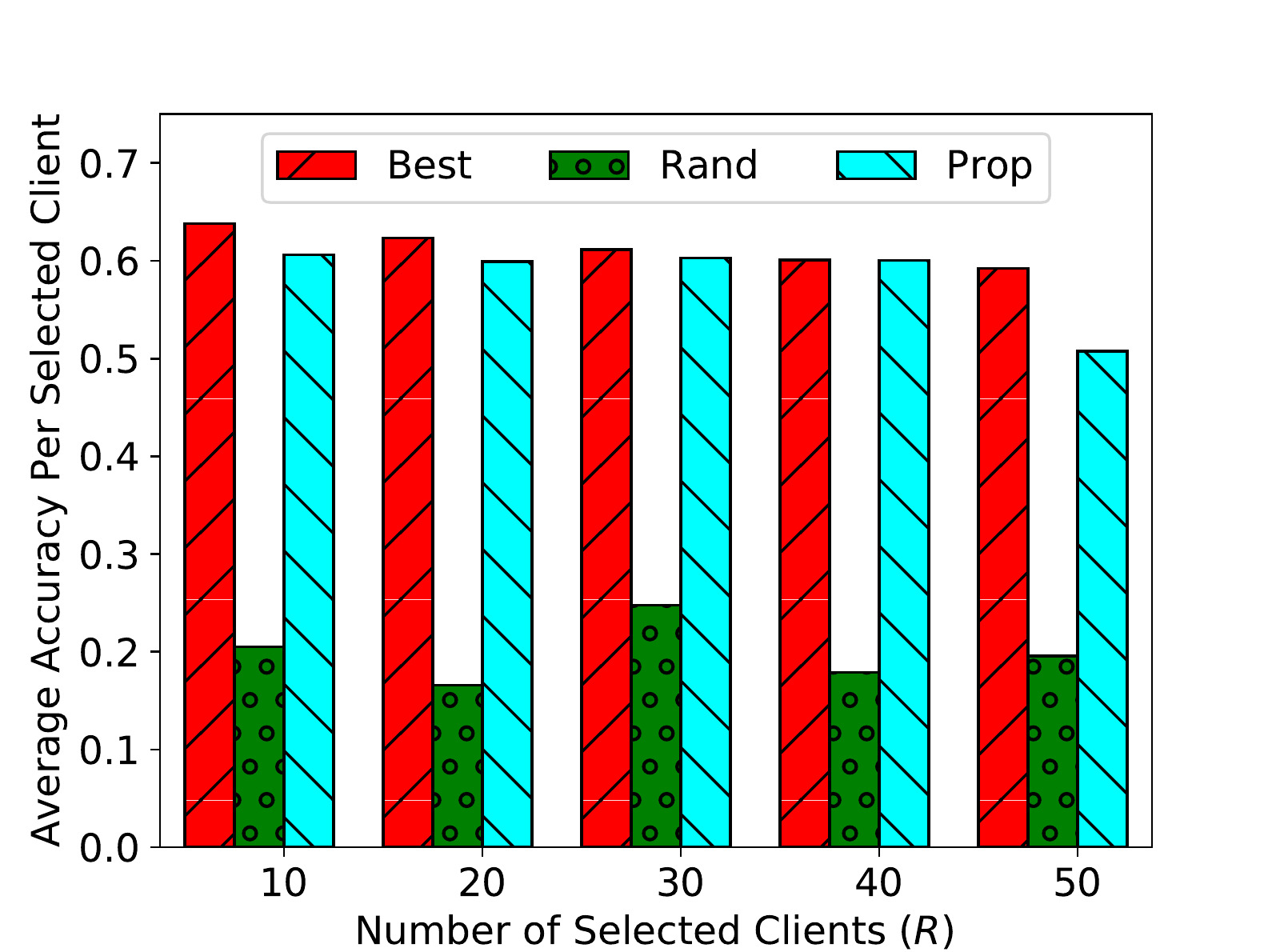}
\caption{Average accuracy per selected client.}
\label{fig:Ac_R}
\end{subfigure}
\begin{subfigure}[b]{0.32\textwidth}
\includegraphics[width=\textwidth]{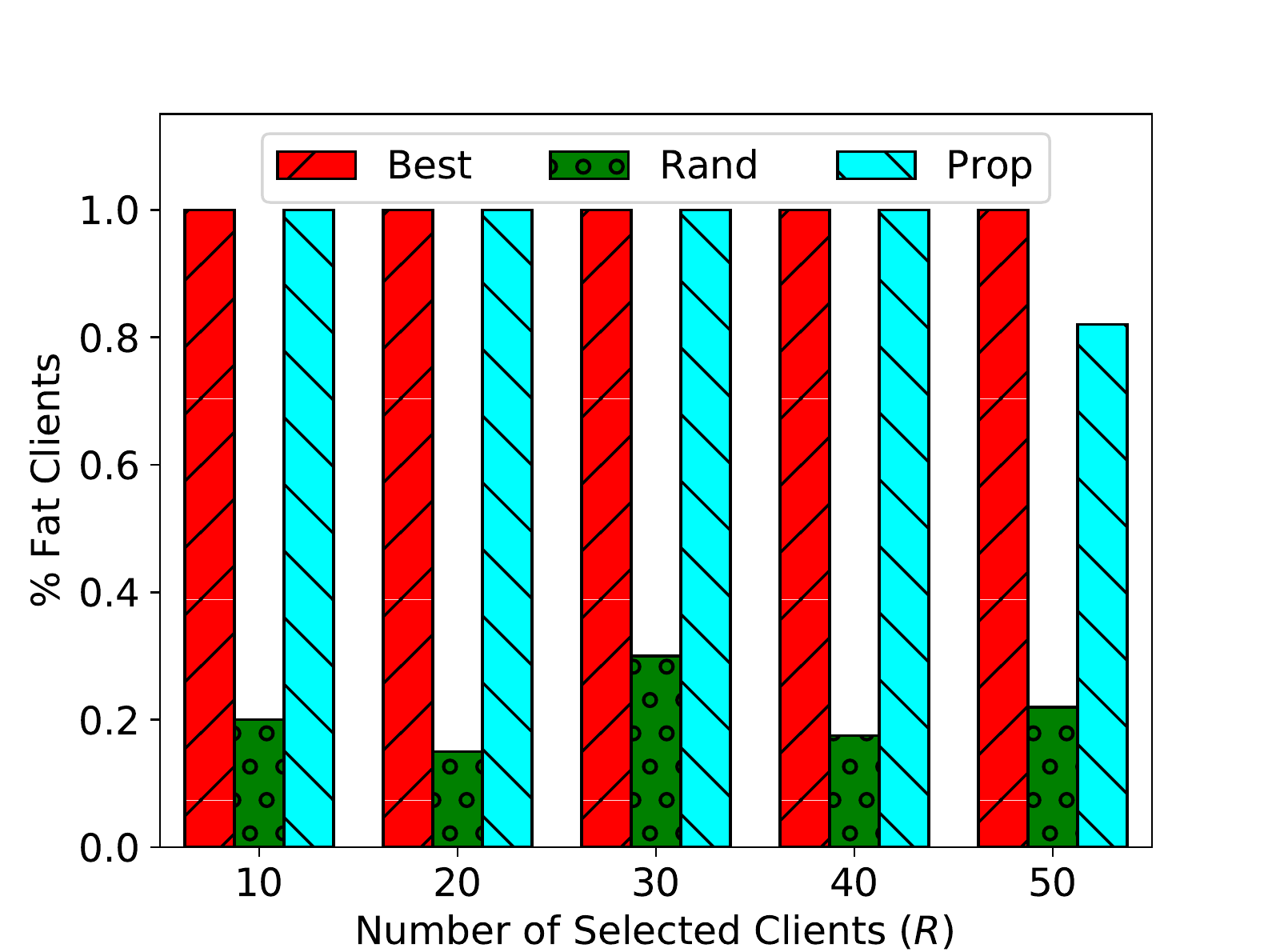}
\caption{\% Fat clients.}
\label{fig:fat_R}
\end{subfigure}
\caption{Performance of algorithms for different $R$, number of selected clients, values (10, 20, 30, 40, 50) while fixing $N$, number of clients, to 400 and $r_2$ to 4 ($\alpha^*$ is 43). \textit{Our proposed algorithm is more competitive for smaller values of $R$ and as $R$ is increased, the performance of algorithms converges.}}
\label{fig:R}
\end{figure*}

\begin{figure*}[htbp]
\centering
\begin{subfigure}[b]{0.32\textwidth}
\includegraphics[width=\textwidth]{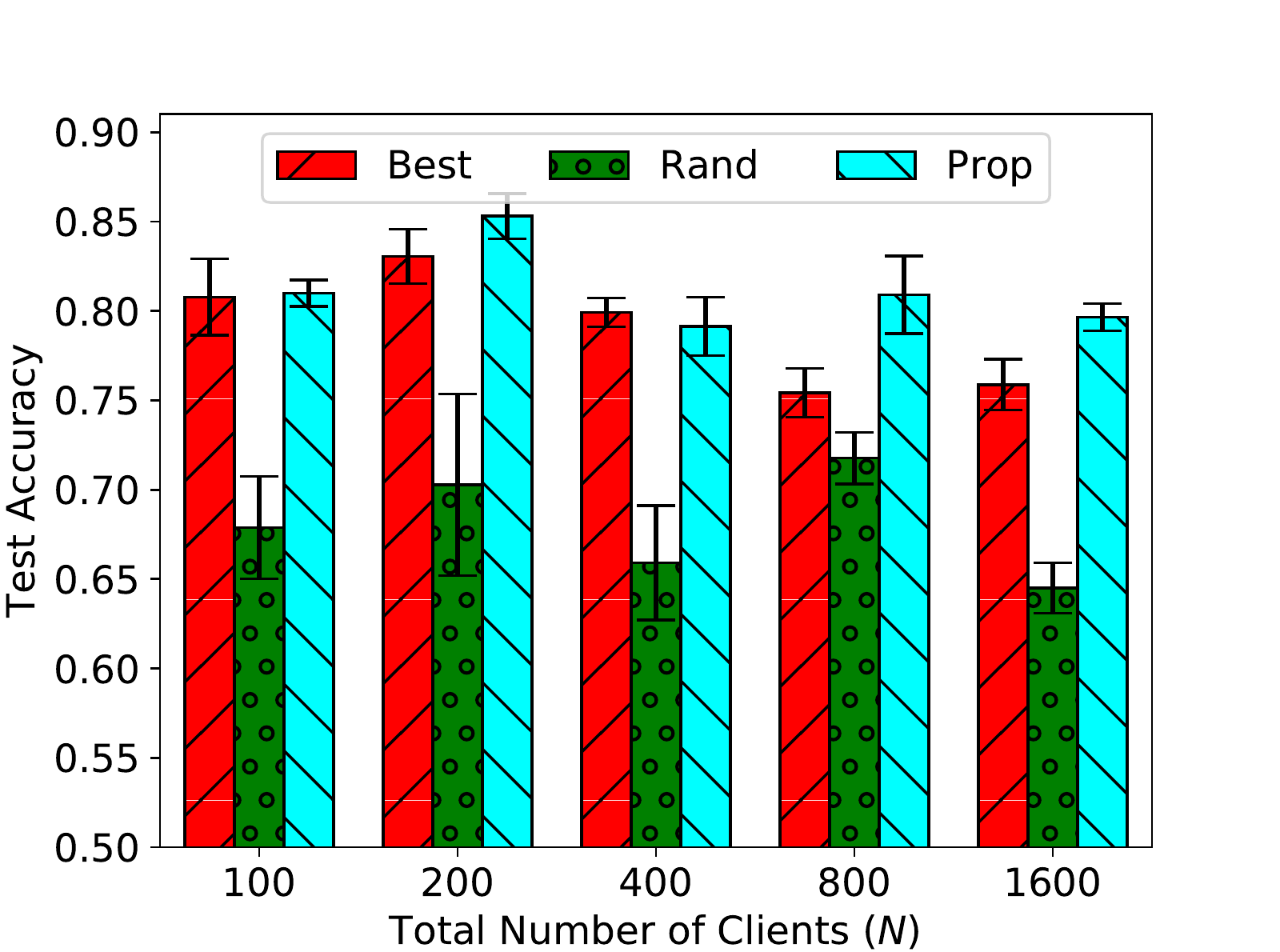}
\caption{Test accuracy.}
\label{fig:A_N}
\end{subfigure}
\begin{subfigure}[b]{0.32\textwidth}
\includegraphics[width=\textwidth]{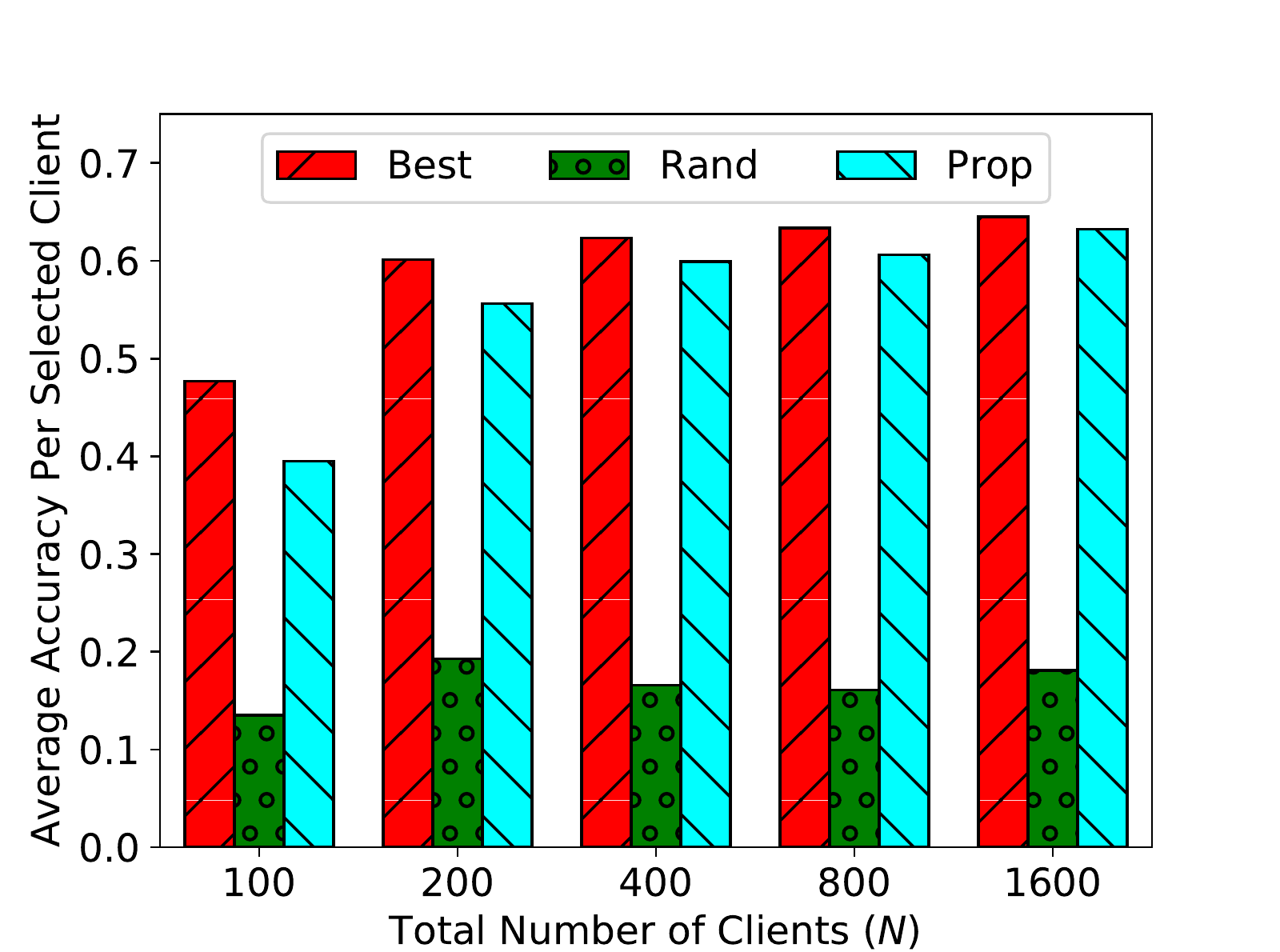}
\caption{Average accuracy per selected client.}
\label{fig:Ac_N}
\end{subfigure}
\begin{subfigure}[b]{0.32\textwidth}
\includegraphics[width=\textwidth]{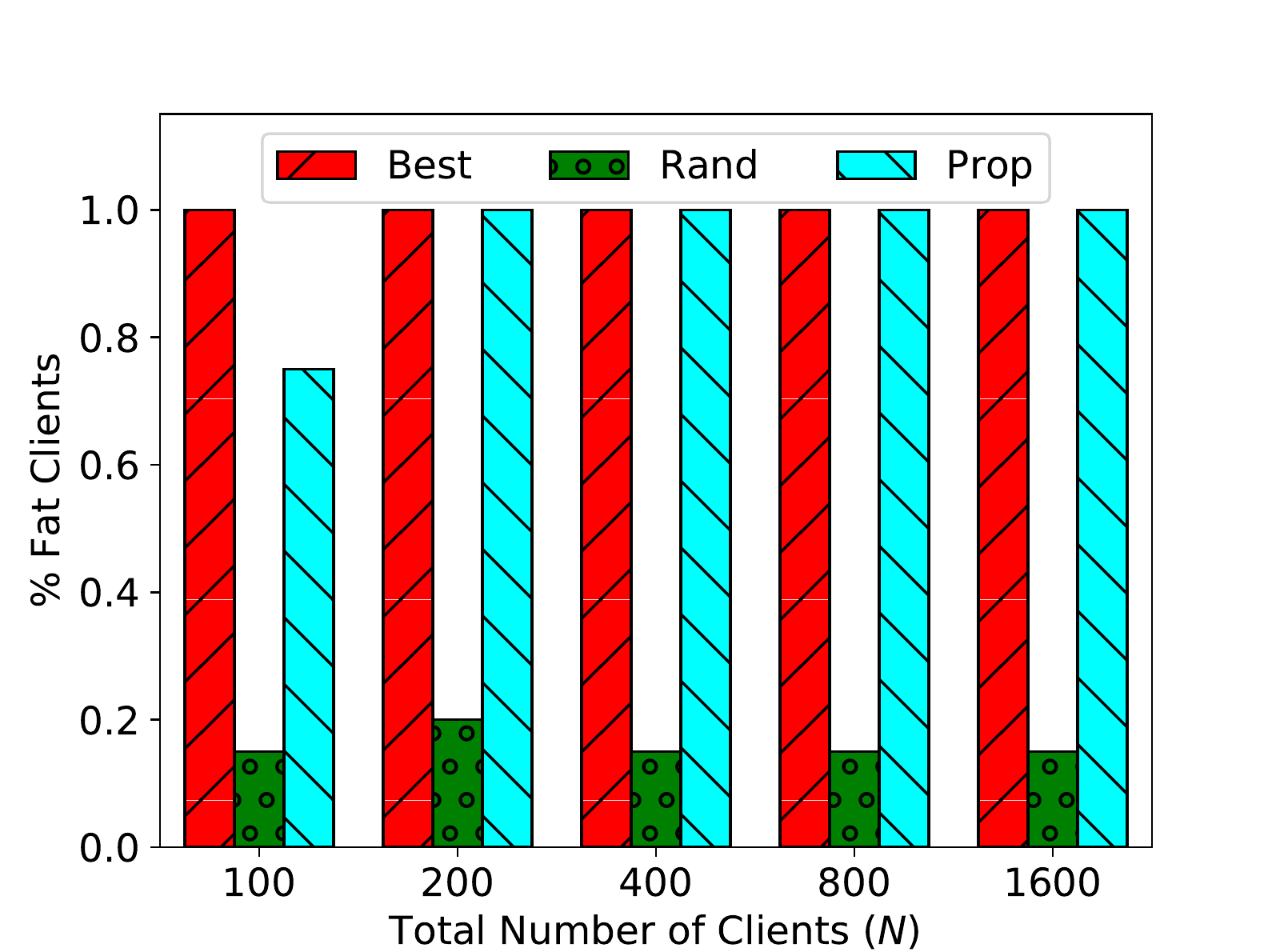}
\caption{\% Fat clients.}
\label{fig:fat_N}
\end{subfigure}
\caption{Performance of algorithms for different $N$, number of clients, values (100, 200, 400, 800, and 1600) while fixing $R$, number of selected clients, to 30 and $r_2$ to 4 (different $\alpha$ per $N$ value). \textit{Our proposed algorithm performs better than the random algorithm approaching the performance of the best algorithm regardless of the value of $N$.}}
\label{fig:N}
\end{figure*}

\section{Discussion}
\label{result_discussion}

Since we cannot present all values and figures of the 125 experiments, we fixed In each case 2 of the variables (i.e., $N$, $R$, and $r_2$) and show the results for changing the third variable. Also, for each case, we present 3 figures: the first represents the general accuracy of the system; the second for the average accuracy of selected candidate clients, which is used to indicate the contribution of individual candidate clients toward the general accuracy of the system; and finally, the percentage of the accepted Fat clients, which is useful for investigating if more Fat clients lead to higher accuracy.

\subsection{Experimenting with different values of $r_2$}

The results illustrated in Fig. \ref{fig:r2} support the discussion in Section \ref{performance_analysis}. Increasing the value of $r_2$ while fixing $r_1$ to 1 reduces the value of $\alpha^*$ and thus increases the probability of finding the best candidate clients as shown in Fig. \ref{fig:A_r2}. Furthermore, Fig. \ref{fig:Ac_r2} and Fig. \ref{fig:Ac_r2} confirms the fact that increasing the value of $r_2$ leads to accepting more fat candidate clients with higher accuracy.

\subsection{Experimenting with different values of $R$}

Fig. \ref{fig:R} shows that the proposed heuristic is more competitive when the number of selected candidate clients (i.e., $R$) is low. However, as $R$ increases, the accuracy of all algorithms converges as indicated in Fig. \ref{fig:A_R}. As the number of selected candidate clients increases, all algorithms will have a good portion of the dataset and will be able to converge to a high accuracy in less time. As a result, there is no problem to solve for high values of $R$. Besides, sometimes it is not feasible to contact many candidate clients since some of them are not available. Figures \ref{fig:Ac_R} and \ref{fig:fat_R} show that the accuracy of the system increases regardless of the accuracy of the individual candidate clients and the number of fat nodes.

\subsection{Experimenting with different values of $N$}

The performance of the proposed heuristic is almost stable when the total number of candidate clients is increased while fixing $R$ to 30 as illustrated in Fig. \ref{fig:N}. The accuracy of the proposed heuristic is almost 80\% for different values of $N$ as shown in Fig. \ref{fig:A_N}. Additionally, Fig. \ref{fig:Ac_N} and Fig. \ref{fig:fat_N} support this argument.

\subsection{Lessons Learned}
\label{lessons_learned}

We can conclude here that based on the presented results:

\begin{itemize}
 \item The performance of the proposed online heuristic is stable regardless of the number of total clients $N$, as illustrated in Fig. \ref{fig:N}.
 
 \item The accuracy of the proposed heuristic increases as the number of selected candidate clients $R$ increases. However, as $R$ goes up, the performance of the proposed online heuristic, the online random algorithm, and the offline best algorithm tends to converge as indicated in Fig. \ref{fig:R}. This is because, with more candidate clients, algorithms have access to a larger portion of the overall dataset.
 
 \item As the number of best candidate clients ($r_2$) is increased, the performance of the proposed online heuristic is enhanced since better candidate clients are used as shown in Fig. \ref{fig:r2}.
\end{itemize}

\section{Conclusions and Future Work}
\label{conclusion_future_work}

In this paper, the problem of optimizing accuracy in stateful federated learning by selecting the best candidate clients based on test accuracy is considered. Then, the problem of maximizing the probability of selecting the best candidate clients based on accuracy is formulated as a secretary problem and performance analysis is presented along with proofs. Based on the formulation, an online stateful federated learning heuristic is proposed to find the best candidate clients. In addition, an IoT client alarm application is proposed that utilizes the proposed heuristic along with IoT device type classification to identify unauthorized IoT devices and alert clients. To test the efficiency of the proposed heuristic, we run many experiments using a real IoT dataset and the performance of the online random algorithm and the offline best algorithm are compared against the performance of the proposed heuristic. Results show that the proposed heuristic performs better than the two state-of-the-art algorithms. Additionally, we notice the stability in the performance of the proposed heuristic compared against the performance of the other two algorithms regardless of the number of participating candidate clients. We also notice that when increasing the number of best selected candidate clients, the proposed heuristic becomes less competitive. This is because with more clients comes more data and thus the performance of algorithms converges regardless of how bad an algorithm in selecting candidate clients.

We want to emphasize a disclaimer that the proposed algorithm is not designed to work efficiently in stateless FL. Moreover, the proposed algorithm is not competitive in applications of offline nature and/or with open budget in terms of selected clients since clients' dataset size and testing accuracies are known in advance and all clients can be used in the training process.

In the future, we plan to devise different variations of the secretary problem and provide performance analysis along with proofs for each. We also intend to run several experiments using a real dataset to evaluate those variations and compare their performance with the performance of the proposed heuristic.

\ifCLASSOPTIONcaptionsoff
 \newpage
\fi

\bibliographystyle{IEEEtran}
\bibliography{biblo}

\vskip -2\baselineskip plus -1fil
\begin{IEEEbiography}[{\includegraphics[width=1in, height=2.5in, clip,keepaspectratio]{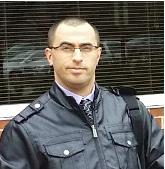}}]{Ihab Mohammed} is an Assistant Professor at the Computer Science Department of Western Illinois University (WIU). He received his B.S. and M.S. degrees in computer science from Al-Nahrain University, Baghdad, Iraq, in 2002 and 2005, respectively. He received his Ph.D. from Western Michigan University in 2020. In 2005, he joined the Department of Computer Science, Al-Nahrain University as a lecturer. In 2014, he joined Western Michigan University as a Ph.D. student and worked in different positions including Research Assistant, Graduate Teaching Assistant, and Part-Time Faculty Instructor. Also, Ihab served as a reviewer for high-quality journals such as IEEE Transaction of Mobile Computing, IEEE Internet of Things Journal, IET Networks, and IET Intelligent Transport Systems. Also, he served as a reviewer for several conferences including IEEE VTC-2018, IEEE ISCC 2018, and IEEE IWCMC 2018. His current research interests include Internet of Things (IoT), Security, Cloud Computing, Data Analysis and Algorithm Design with Machine Learning, Deep Learning, and Federated Learning.
\end{IEEEbiography}
\vskip -2\baselineskip plus -1fil

\begin{IEEEbiography}[{\includegraphics[width=1in, height=2.5in, clip,keepaspectratio]{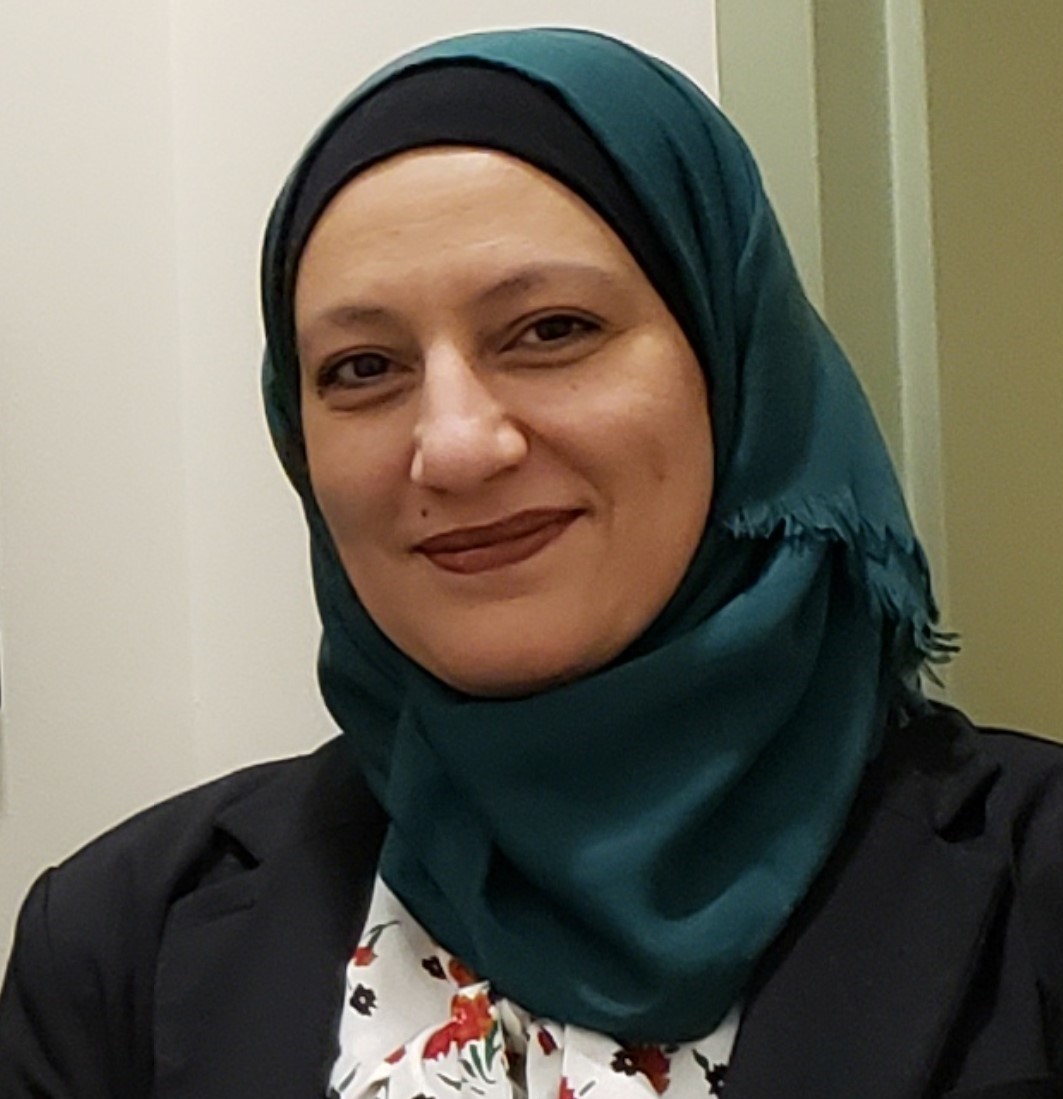}}]{Shadha Tabatabai} is a Ph.D. candidate at the CS department of Western Michigan University (WMU). She received her B.S. and M.S. degrees in computer science from Al-Nahrain University, Baghdad, Iraq, in 1998 and 2001, respectively. Shadha joined the CS Department at Al-Nahrain University n 2002 as an Assistant Lecturer and promoted to a Lecturer later in 2012. Additionally, she served as Research Assistant, Graduate Teaching Assistant, and Part-Time Faculty Instructor for the CS department at WMU. Moreover, she served as a reviewer for different conferences including IEEE ISCC 2018, IEEE IWCMC 2018, and IEEE VTC-2018. Her current research interests include Internet of Things, cloud computing, machine learning, and reinforcement learning.
\end{IEEEbiography}
\vskip -2\baselineskip plus -1fil

\begin{IEEEbiography}[{\includegraphics[width=1in, height=1.25in,clip,keepaspectratio]{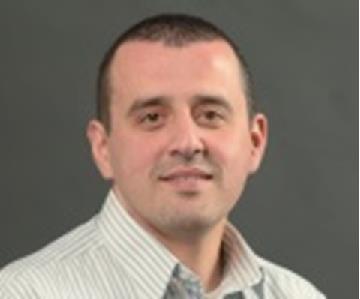}}]{Ala Al-Fuqaha} (S'00-M'04-SM'09) (S'00-M'04-SM'09) received Ph.D. degree in Computer Engineering and Networking from the University of Missouri-Kansas City, Kansas City, MO, USA. He is currently a professor at the Information and Computing Technology division, college of Science and Engineering, Hamad Bin Khalifa University (HBKU) and the Computer Science department, Western Michigan University. His research interests include the use of machine learning in general and deep learning in particular in support of the data-driven and self-driven management of large-scale deployments of IoT and smart city infrastructure and services, Wireless Vehicular Networks (VANETs), cooperation and spectrum access etiquette in cognitive radio networks, and management and planning of software defined networks (SDN). He is a senior member of the IEEE and an ABET Program Evaluator (PEV). He serves on editorial boards of multiple journals including IEEE Communications Letter and IEEE Network Magazine. He also served as chair, co-chair, and technical program committee member of multiple international conferences including IEEE VTC, IEEE Globecom, IEEE ICC, and IWCMC.
\end{IEEEbiography}
\vskip -2\baselineskip plus -1fil

\begin{IEEEbiography}[{\includegraphics[width=1in, height=1.25in,clip,keepaspectratio]{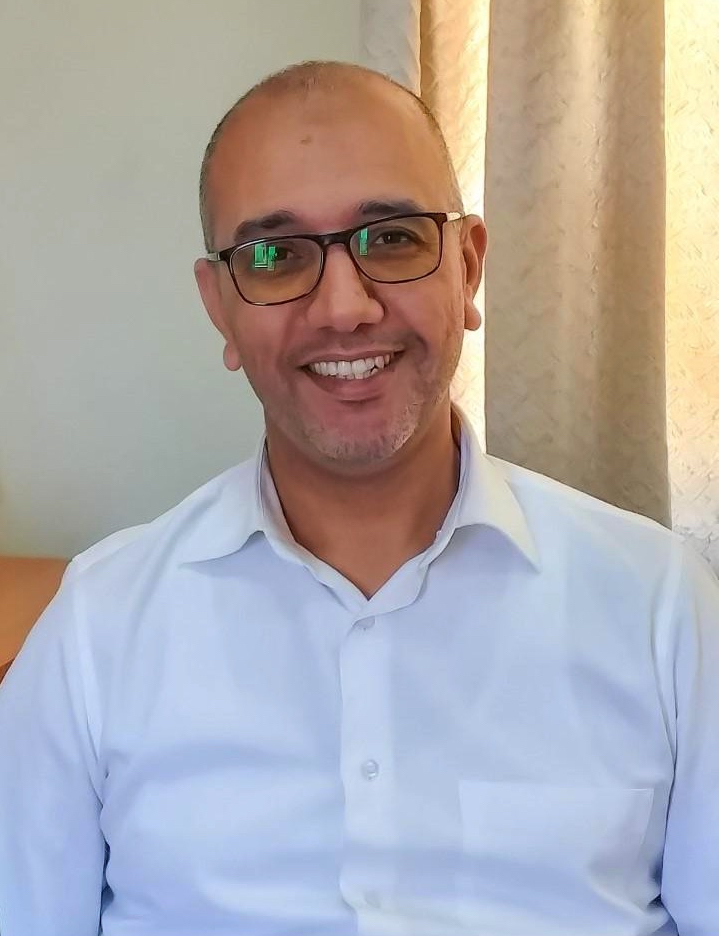}}]{Faissal El Bouanani} was born in Nador, Morocco, in 1974. He received the M.S. and Ph.D. degrees in network and communication engineering from Mohammed V University-Souissi, Rabat, Morocco, in 2004 and 2009, respectively. He has served as a Faculty Member with the University of Moulay Ismail, Meknes, from 1997 to 2009, before joining the National High School of IT/ENSIAS College of Engineering, Mohammed V University, Rabat, in 2009, where he is currently an Associate Professor. He advised many Ph.D. and master's students at both Mohammed V and Moulay Ismail Universities. So far, his research efforts have culminated in more than 75 papers in a wide variety of international conferences and journals. His current research interests include performance analysis and design of wireless communication systems. He has also been involved as a TPC Member in various conferences and IEEE journals. He is also an Associate Editor of IEEE ACCESS and Editor of Frontiers in Communications and Networks journals. His Ph.D. thesis was awarded the best one by Mohammed V University-Souissi, in 2010. He served as the TPC Chair of the ICSDE conferences and the General Co-Chair of ACOSIS'16 and CommNet'18 conferences. He serves as the General Chair of the 2019/2020 CommNet conferences. Dr. El Bouanani is also a Senior Member of the IEEE.
\end{IEEEbiography}
\vskip -2\baselineskip plus -1fil

\begin{IEEEbiography}[{\includegraphics[width=1in, height=1.25in,clip,keepaspectratio]{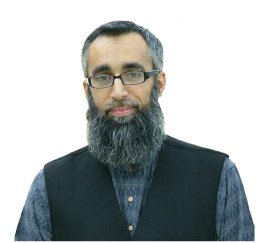}}]{Junaid Qadir} is a Professor at the Information Technology University (ITU), Lahore, Pakistan, where he also serves as the Chairperson of the Electrical Engineering Department. He is the Director of the IHSAN Research Lab at ITU (http://ihsanlab.itu.edu.pk/) since December 2015. He completed Ph.D. from University of New South Wales, Australia in 2008 and his Bachelor's in Electrical Engineering from University of Engineering and Technology (UET), Lahore, Pakistan in 2000. His primary research interests are in the areas of computer systems and networking, applied machine learning, using ICT for development (ICT4D); and engineering education. He has published more than 100 peer-reviewed articles at various high-quality research venues including more than 50 impact-factor journal publications at top international research journals including IEEE Communication Magazine, IEEE Journal on Selected Areas in Communication (JSAC), IEEE Communications Surveys and Tutorials (CST), and IEEE Transactions on Mobile Computing (TMC). He is an award-winning teacher who has been awarded the highest national teaching award in Pakistan—the higher education commission's (HEC) best university teacher award—for the year 2012-2013. He is an ACM Distinguished Speaker since January 2020. He is a senior member of both IEEE and ACM.
\end{IEEEbiography}
\vskip -2\baselineskip plus -1fil

\begin{IEEEbiography}[{\includegraphics[width=1in, height=1.25in,clip,keepaspectratio]{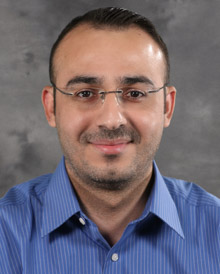}}]{Basheer Qolomany} (S'17) received the Ph.D. and second master's en-route to Ph.D. degrees in Computer Science from Western Michigan University (WMU), Kalamazoo, MI, USA, in 2018. He also received his B.Sc. and M.Sc. degrees in computer science from University of Mosul, Mosul city, Iraq, in 2008 and 2011, respectively. He is currently an Assistant Professor at Department of Cyber Systems, University of Nebraska at Kearney (UNK), Kearney, NE, USA. Previously, he served as a Visiting Assistant Professor at Department of Computer Science, Kennesaw State University (KSU), Marietta, GA, USA, in 2018-2019; a Graduate Doctoral Assistant at Department of Computer Science, WMU, in 2016-2018; he also served as a lecturer at Department of Computer Science, University of Duhok, Kurdistan region of Iraq, in 2011-2013. His research interests include machine learning, deep learning, Internet of Things, smart services, cloud computing, and big data analytics. Dr. Qolomany has served as a reviewer of multiple journals, including IEEE Internet of Things journal, Energies — Open Access Journal, and Elsevier - Computers and Electrical Engineering journal. He also served as a Technical Program Committee (TPC) member and a reviewer of some international conferences including IEEE Globecom, IEEE IWCMC, and IEEE VTC.
\end{IEEEbiography}
\vskip -2\baselineskip plus -1fil

\begin{IEEEbiography}[{\includegraphics[width=1in, height=1.25in,clip,keepaspectratio]{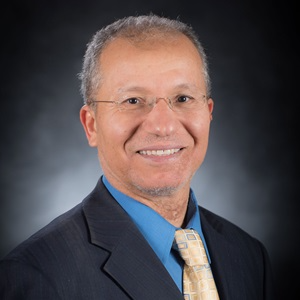}}]%
{Mohsen Guizani (S'85-M'89-SM'99-F'09)}
received the B.S. and M.S. degrees in electrical engineering and the M.S. and Ph.D. degrees in computer engineering from Syracuse University, in 1984, 1986, 1987, and 1990, respectively. He was the Associate Vice President of Qatar University, the Chair of the Computer Science Department, Western Michigan University, the Chair of the Computer Science Department, University of West Florida, and the Director of graduate studies at the University of Missouri–Columbia. He is currently a Professor with the Department of Computer Science and Engineering, Qatar University. He has authored or coauthored nine books and publications in refereed journals and conferences. His research interests include wireless communications and mobile computing, vehicular communications, smart grid, cloud computing, and security.
\end{IEEEbiography}

\end{document}